\documentclass{article}

\PassOptionsToPackage{numbers, compress}{natbib}

\PassOptionsToPackage{sort}{natbib}

\usepackage[preprint]{Styles/neurips_2026}

\usepackage[utf8]{inputenc} 
\usepackage[T1]{fontenc}    
\usepackage{hyperref}       
\usepackage{url}            
\usepackage{booktabs}       
\usepackage{amsfonts}       
\usepackage{nicefrac}       
\usepackage{microtype}      
\usepackage{xcolor}         

\title{Gradient Regularized Newton Boosting Trees with Global Convergence}

\author{%
  \normalfont
  \begin{tabular*}{\dimexpr\textwidth-2\tabcolsep\relax}{@{}@{\extracolsep{\fill}}cccc@{}}
    \textbf{Nikita Zozoulenko}\textsuperscript{1,*,$\dagger$} &
    \textbf{Daniel Falkowski}\textsuperscript{2,*} &
    \textbf{Thomas Cass}\textsuperscript{1} &
    \textbf{Lukas Gonon}\textsuperscript{1,3}
  \end{tabular*} \\[1ex]
  {\small \textsuperscript{1}Imperial College London \qquad
  \textsuperscript{2}Stockholm University \qquad
  \textsuperscript{3}University of St.\ Gallen} \\[0.5ex]
  {\small *Equal contribution \qquad \textsuperscript{$\dagger$}Corresponding author:
  \texttt{n.zozoulenko23@imperial.ac.uk}}
}

\usepackage{amsmath, amssymb, amsthm}
\usepackage{hyperref}
\usepackage[capitalize,nameinlink]{cleveref}
\usepackage[most]{tcolorbox}

\usepackage{algorithm}
\usepackage{algorithmic}

\usepackage{placeins}
\usepackage{caption}
\definecolor{mycyan}{RGB}{0, 160, 180}   
\definecolor{myblue}{RGB}{0, 0, 150}     
\definecolor{mypurple}{RGB}{120, 0, 150} 
\definecolor{myred}{RGB}{210, 0, 60}     
\definecolor{myorange}{RGB}{200, 80, 0}  
\definecolor{mygray}{RGB}{100, 100, 100} 
\definecolor{proofgray}{RGB}{245, 245, 245}

\newcounter{sharedthm}[section]
\renewcommand{\thesharedthm}{\thesection.\arabic{sharedthm}}

\crefname{definition}{Definition}{Definitions}
\crefname{proposition}{Proposition}{Propositions}
\crefname{theorem}{Theorem}{Theorems}
\crefname{lemma}{Lemma}{Lemmas}
\crefname{assumption}{Assumption}{Assumptions}
\crefname{remark}{Remark}{Remarks}
\crefname{example}{Example}{Examples}

\tcbset{
    basebox/.style={
        enhanced, breakable, frame hidden,
        left=3pt, right=0pt, top=-3pt, bottom=-3pt,
        attach title to upper={\ }, 
        fonttitle=\bfseries,
        sharp corners,
        colback=white,    
        colframe=white,   
    }
}

\newtcolorbox{innerdef}[1][]{
    basebox, coltitle=mycyan, borderline west={2pt}{0pt}{mycyan},
    code={\crefalias{sharedthm}{definition}\refstepcounter{sharedthm}},
    #1
}
\newtcolorbox{innerprop}[1][]{
    basebox, coltitle=myblue, borderline west={2pt}{0pt}{myblue},
    code={\crefalias{sharedthm}{proposition}\refstepcounter{sharedthm}},
    #1
}
\newtcolorbox{innerthm}[1][]{
    basebox, coltitle=myblue, borderline west={2pt}{0pt}{myblue},
    code={\crefalias{sharedthm}{theorem}\refstepcounter{sharedthm}},
    #1
}
\newtcolorbox{innercor}[1][]{
    basebox, coltitle=myblue, borderline west={2pt}{0pt}{myblue},
    code={\crefalias{sharedthm}{corollary}\refstepcounter{sharedthm}},
    #1
}
\newtcolorbox{innerlem}[1][]{
    basebox, coltitle=mypurple, borderline west={2pt}{0pt}{mypurple},
    code={\crefalias{sharedthm}{lemma}\refstepcounter{sharedthm}},
    #1
}
\newtcolorbox{innerass}[1][]{
    basebox, coltitle=myred, borderline west={2pt}{0pt}{myred},
    code={\crefalias{sharedthm}{assumption}\refstepcounter{sharedthm}},
    #1
}
\newtcolorbox{innerrem}[1][]{
    basebox, coltitle=myorange, borderline west={2pt}{0pt}{myorange},
    code={\crefalias{sharedthm}{remark}\refstepcounter{sharedthm}},
    #1
}
\newtcolorbox{innerex}[1][]{
    basebox, coltitle=mygray, borderline west={2pt}{0pt}{mygray},
    code={\crefalias{sharedthm}{example}\refstepcounter{sharedthm}},
    #1
}

\newenvironment{definition}[1][]
  {\if\relax\detokenize{#1}\relax
     \begin{innerdef}[title={Definition \thesharedthm.}]
   \else
     \begin{innerdef}[title={Definition \thesharedthm\ (#1).}]
   \fi}
  {\end{innerdef}}

\newenvironment{proposition}[1][]
  {\if\relax\detokenize{#1}\relax
     \begin{innerprop}[title={Proposition \thesharedthm.}]
   \else
     \begin{innerprop}[title={Proposition \thesharedthm\ (#1).}]
   \fi}
  {\end{innerprop}}

\newenvironment{theorem}[1][]
  {\if\relax\detokenize{#1}\relax
     \begin{innerthm}[title={Theorem \thesharedthm.}]
   \else
     \begin{innerthm}[title={Theorem \thesharedthm\ (#1).}]
   \fi}
  {\end{innerthm}}

\newenvironment{corollary}[1][]
  {\if\relax\detokenize{#1}\relax
     \begin{innercor}[title={Corollary \thesharedthm.}]
   \else
     \begin{innercor}[title={Corollary \thesharedthm\ (#1).}]
   \fi}
  {\end{innercor}}

\newenvironment{lemma}[1][]
  {\if\relax\detokenize{#1}\relax
     \begin{innerlem}[title={Lemma \thesharedthm.}]
   \else
     \begin{innerlem}[title={Lemma \thesharedthm\ (#1).}]
   \fi}
  {\end{innerlem}}

\newenvironment{assumption}[1][]
  {\if\relax\detokenize{#1}\relax
     \begin{innerass}[title={Assumption \thesharedthm.}]
   \else
     \begin{innerass}[title={Assumption \thesharedthm\ (#1).}]
   \fi}
  {\end{innerass}}

\newenvironment{remark}[1][]
  {\if\relax\detokenize{#1}\relax
     \begin{innerrem}[title={Remark \thesharedthm.}]
   \else
     \begin{innerrem}[title={Remark \thesharedthm\ (#1).}]
   \fi}
  {\end{innerrem}}

\DeclareMathOperator*{\argmin}{argmin}

\usepackage[textsize=tiny]{todonotes}

\usepackage{mathrsfs}
\usepackage{listings}
\usepackage{xcolor}

\newcommand{\bbR}{\mathbb{R}}

\newcommand{\E}{\mathbb{E}}

\usepackage{wrapfig}     
\usepackage{float}       

\begin{document}

\maketitle

\begin{abstract}
    Gradient Boosting Decision Trees (GBDTs) dominate tabular machine learning, with modern implementations like XGBoost, LightGBM, and CatBoost being based on Newton boosting: a second-order descent step in the space of decision trees. Despite its empirical success, the global convergence of Newton boosting is poorly understood compared to first-order boosting. In this paper, we introduce Restricted Newton Descent, a framework for convex optimization with Newton's method on Hilbert spaces with inexact iterates, based on the concepts of cosine angle and weak gradient edge.  Within this framework, we recover Newton boosting with GBDTs and classical finite-dimensional theory as special cases. We first prove that vanilla Newton boosting achieves a linear rate of convergence for smooth, strongly convex losses that satisfy a Hessian-dominance condition. To handle general convex losses with Lipschitz Hessians, we extend a recent gradient regularized Newton scheme to the restricted weak learner setting, establishing the first global $\mathcal{O}(\frac{1}{k^2})$ convergence rate for second-order GBDTs. This scheme minimally modifies the classical algorithm by introducing an adaptive $\ell_2$-regularization term proportional to the square root of the gradient norm at each iteration.
    In numerical experiments, we show that this scheme converges while vanilla Newton boosting may diverge.
\end{abstract}

\section{Introduction}\label{sectionIntroduction}

Gradient Boosting Decision Trees (GBDTs) are powerful machine learning models, excelling in tabular settings such as regression, classification, and learning to rank over neural baselines \cite{
2021AreNeuralRankersStillOutperformedByGradientBoostedDecisionTrees,
2022WhyDoTreesOutperformDeepLearningTabularData,
2022TabularDataDeepLearningIsNotAllYouNeed,
2023WhenDoNeuralNetsOutperformBoostedTreesOnTabularData,
2025TabArenaBenchmarkTabularData}. Modern implementations such as XGBoost \cite{2016XGBoost}, LightGBM \cite{2017LightGBM}, and CatBoost \cite{2018CatBoost} are based on the Newton boosting algorithm, which updates a decision tree ensemble by performing a Newton descent step in the space of decision trees, derived via a second order Taylor expansion with respect to the training data \cite{2016XGBoost}. However, despite the widespread use of GBDTs, the theoretical mechanisms of Newton boosting remain poorly understood, with most prior work focused on first-order gradient boosting \cite{1999BoostingAlgorithmsAsGradientDescent,
2001GreedyFunctionApproximationAGradientBoostingMachine,
2011GeneralizedBoostingAlgorithmsForConvexOptimization,
2017ANewPerspectiveOnBoostingInLinearRegressionViaSubgradientOptimization,
2019RegularizedGradientBoosting,
2020RandomizedGradientBoostingMachine,
2020AcceleratingGradientBoostingMachines}.

While Newton's method in $\mathbb{R}^N$ offers fast quadratic local convergence, it lacks global convergence guarantees for general strictly convex losses, even when combined with line search \cite{2016SimpleExamplesForTheFailureOfNewtonsMethodWithLineSearchForStrictlyConvexMinimization}. Recently, the convex optimization literature has seen a resurgence of interest in regularized Newton schemes that achieve global $\mathcal{O}(\frac{1}{k^2})$ convergence rates. This was first established for the Cubic Regularized Newton (CRN) method \cite{2006CubicRegularizationOfNewtonMethodAndItsGlobalPerformance, 2008AcceleratingTheCubicRegularizationOfNewtonsMethodOnConvexProblems}; however, CRN requires solving an expensive subproblem at each iteration, making it unfit for GBDTs. More recently, the Gradient Regularized Newton (GRN) method was proposed, achieving the same $\mathcal{O}(\frac{1}{k^2})$ rate as CRN \cite{2023RegularizedNewtonMethodWithGlobalConvergence, 2024GradientRegularizationOfNewtonMethodWithBregmanDistances}. GRN uses an adaptive $\ell_2$-regularization term proportional to the square root of the gradient norm at each iteration. Crucially, this method avoids expensive subproblems and incurs negligible computational overhead compared to the standard Newton update. This discovery has sparked new research into the global convergence of second-order numerical schemes \cite{2022ADampedNewtonMethodAchievesGlobalMinus2AndLocalQuadraticConvergenceRate, 2024SuperUniversalRegularizedNewtonMethod, 2025MinimizingQuasiSelfConcordantFunctionsByGradientRegularizationOfNewtonMethod, 2025SketchAndProjectMeetsNewtonMethodGlobalConvergenceWithLowRankUpdates}. A natural question arises: can these convergence guarantees be extended to GBDTs and optimization with weak learners, providing convergence rates for second-order boosting methods used in practice?

\textbf{Contributions:}
In this paper, we formulate Newton boosting as a convex optimization problem on a Hilbert space $\mathcal{H}$, where optimization is performed on a restricted set of permissible weak learners $\mathcal{F} \subset \mathcal{H}$. Our contributions are:

\begin{itemize}
    \item We introduce the \textbf{Restricted Newton Descent} framework, based 
    on the concepts of weak and exact iterates, cosine angles, and weak gradient edge. This recovers both Newton boosting with GBDTs and classical finite-dimensional convex optimization as special cases.

    \item For smooth strongly convex losses satisfying a Hessian-dominance condition, 
    we prove that vanilla Restricted Newton Descent achieves a linear
    convergence rate. This includes practically relevant losses such as 
    $\ell_2$-regularized binary and categorical cross-entropy.

\item For general convex losses with $2M$-Lipschitz Hessians, we generalize 
    the GRN scheme of 
    \citet{2023RegularizedNewtonMethodWithGlobalConvergence} to the Hilbert space 
    setting with weak learners, establishing the first global 
    $\mathcal{O}\!\left(\frac{1}{k^2}\right)$ convergence rate for second-order 
    GBDTs, valid for all learning rates $\eta \leq 1$. If the loss is additionally 
    strongly convex, a local linear rate is also obtained.



    
    \item We present numerical results showing that the gradient regularized Newton scheme converges while vanilla Newton boosting may diverge.    
\end{itemize}

\subsection{Related Literature}

\textbf{GBDTs:} Gradient boosting models have traditionally been viewed as gradient descent in $\mathbb{R}^N$ on the training data, where $N$ is the dataset size \cite{1999BoostingAlgorithmsAsGradientDescent, 2000LogitBoostAdditiveLogisticRegressionAStatisticalViewOfBoosting}. Later works identified this process with gradient descent in the Hilbert space of square-integrable functions $\mathcal{L}^2(\widehat{\nu}_N)$, where $\widehat{\nu}_N$ is the empirical data distribution \cite{2011GeneralizedBoostingAlgorithmsForConvexOptimization}; see also connections to gradient representation boosting in deep learning \cite{2018FunctionalGradientBoostingResNet, 2020GeneralizedBoosting, 2025RandomFeatureRepresentationBoosting}. Recent theoretical work has shifted from the gradient perspective to studying boosting via coordinate descent in high-dimensional spaces, achieved by enumerating all possible decision trees \cite{2019RegularizedGradientBoosting, 2020RandomizedGradientBoostingMachine, 2020SnapBoostAHeterogeneousBoostingMachine}.

\textbf{Newton Boosting:} Newton-like updates were first introduced by LogitBoost \cite{2000LogitBoostAdditiveLogisticRegressionAStatisticalViewOfBoosting} for logistic loss, and later applied to GBDTs in a general context by XGBoost \cite{2016XGBoost}. Empirical differences between first-order and Newton boosting were investigated in \cite{2021GradientAndNewtonBoostingForClassificationAndRegression}. Theoretical analysis of Newton boosting remains limited compared to first-order methods; Notable work includes the LogitBoost analysis \cite{2014AConvergenceRateAnalysisForLogitBoostMARTAndTheirVariant} using probability clamping, and later the SnapBoost \cite{2020SnapBoostAHeterogeneousBoostingMachine} using coordinate descent.

\textbf{Convergence Rates in Boosting:} For second order boosting, both \citet{2014AConvergenceRateAnalysisForLogitBoostMARTAndTheirVariant} and \citet{2020SnapBoostAHeterogeneousBoostingMachine} proved a linear rate in the strongly convex setting. In contrast, our work also covers general convex losses with Lipschitz Hessians. As for first order boosting, the earliest explicit convergence rate was proved in \cite{2006SomeTheoryForGeneralizedBoostingAlgorithms} with a slow $\mathcal{O}((\log k)^{-\frac{1}{2}})$ rate. For smooth convex losses, \citet{2014AConvergenceRateAnalysisForLogitBoostMARTAndTheirVariant} proved a classical $\mathcal{O}(\frac{1}{k})$ rate.  \citet{2020AcceleratingGradientBoostingMachines} incorporated Nesterov momentum to achieve an accelerated $\mathcal{O}(\frac{1}{k^2})$ rate, similar to our work. Linear rates for strongly convex losses were proven in \cite{2011GeneralizedBoostingAlgorithmsForConvexOptimization} using gradient descent, and in \cite{2020RandomizedGradientBoostingMachine} via randomized coordinate descent.  \citet{2012APrimalDualConvergenceAnalysisOfBoosting} studied the primal-dual structure of boosting, obtaining a linear rate with non-standard constants.

\section{Newton Boosting as Hilbert Space Convex Optimization}
In this section, we recall Newton's method in the Hilbert space setting, and introduce the Restricted Newton Descent framework where Newton boosting with GBDTs is recovered as a special case.

\subsection{Hilbert Space Newton's Method}

Let $\mathcal{H}$ be a Hilbert space, and $L : \mathcal{H} \to \mathbb{R}$ be a convex, twice Fréchet differentiable loss function. We aim to minimize $L$ with Hilbert space Newton's method with a fixed learning rate $\eta \in (0,1]$. At each iteration $k$, we update an iterate $F_k = F_{k-1} + \eta f_k = F_0 + \sum_{t=1}^k \eta f_t$ by minimizing a second-order Taylor expansion of the loss:
\begin{align*}
    L(F_k + f) 
    \;\approx\;
    L(F_k) + \left\langle f,g_k + \frac{1}{2}H_k[f]  \right\rangle
    \;=:\; 
    L(F_k) + Q_k(f)
\end{align*}
where $g_k = \nabla L(F_k)$ is the gradient, $H_k = \nabla^2 L(F_k)$ is the Hessian operator, and $Q_k(f)$ is the local \newpage
second-order approximation of $L$ at the point $F_k$, that is after $k$ iter\-ations of the optimization scheme. 
Differentiation yields
\(
\nabla Q_k(f) = g_k + H_k[f]
\).
Setting the derivative to 0 and solving for $f$ gives the Newton update, assuming that $H_k$ is invertible at the point $g_k$:
\begin{align}\label{eqExactNewtonUpdate}
    f_{k+1}  := - H_k^{-1}[g_k] = \argmin_{f \in \mathcal{H}} Q_k(f).
\end{align}

\vspace{-15pt}
\subsection{Restricted Newton Descent}
In the context of gradient boosting, the exact update $f_{k+1}$ in \eqref{eqExactNewtonUpdate} is generally not attainable since weak learners, as opposed to exact Newton directions, are used to minimize the loss. Consequently, the ensemble $F_k$ instead takes the form $F_k := F_0 +\sum_{t=1}^k \eta f_t^w$ for some notion of a weak iterate $f^w_t$. To formalize this, we introduce the \textbf{Restricted Newton Descent} framework (not to be confused with constrained optimization). This differs from standard Newton's method in two ways. Firstly, the gradient $g_k$ and Hessian $H_k$ are computed w.r.t. the sum of weak, rather than exact, iterates $F_k$. Secondly, we do not minimize $Q_k(f)$ over the entire space $f \in \mathcal{H}$, but rather restrict the weak update $f_{k+1}^w$ to belong to some subset $f_{k+1}^w \in \mathcal{F} \subset \mathcal{H}$ of admissible weak learners. More specifically, given a family $\mathcal{F}$, we define the weak iterates $f_{k+1}^w$ as any element
\begin{align}\label{eqWeakNewtonUpdate}
    f_{k+1}^w  \in  \argmin_{f \in \mathcal{F}} Q_k(f).
\end{align}
Note that $f_{k+1}^w$ is not necessarily unique. We consider general families $\mathcal{F}$ satisfying the following scalability condition, first introduced in \citet{2020AcceleratingGradientBoostingMachines}:
\begin{assumption}
    The family $\mathcal{F}$ is closed under scalar multiplication. That is, if $f \in \mathcal{F}$ then $a f \in \mathcal{F}$ for all $a\in\mathbb{R}$.
\end{assumption}

\begin{wrapfigure}{r}{0.50\textwidth}
  \vspace{-23pt} 
  \begin{minipage}{\linewidth} 
    \begin{algorithm}[H]
    \caption{Vanilla Restricted Newton}
    \label{alg:boostingVanillaNewton}
    \begin{algorithmic}[1]
        \STATE \textbf{Input:} Loss $L:\mathcal{H}\to\bbR$. Learning rate $\eta$. Weak learner family $\mathcal{F}$. Initial value $F_0$.
        \FOR{$k = 0, 1, 2, \dots$}
            \STATE $g_k, \: H_k \leftarrow \nabla L(F_k), \: \nabla^2L(F_k)$
            \STATE $\big($Exact target $f_{k+1} \leftarrow - H_k^{-1}g_k \big)$
            \STATE Weak learner $f^{w}_{k+1} \leftarrow \argmin_{f \in \mathcal{F}} Q_k(f)$
            \STATE $F_{k+1} \leftarrow F_k + \eta f^{w}_{k+1}$
        \ENDFOR
    \end{algorithmic}
    \end{algorithm}
  \end{minipage}
  \vspace{-15pt} 
\end{wrapfigure}

Decision trees, commonly used in boosting, satisfy this assumption. If $\mathcal{F} = \mathcal{H}$, then the weak update $f^w_{k+1}$ coincides with the exact update $f_{k+1}$, and standard Newton's method is recovered. The full algorithm is detailed in \cref{alg:boostingVanillaNewton}, with the exact step in parentheses.

Theoretical analysis of the optimization scheme now depends on the error between the chosen weak iterate $f^w_{k+1}$ and the exact iterate $f_{k+1}$,
as this difference introduces a component orthogonal to the weak update in the inner product induced by the quadratic objective. As we show later, different schemes rely on different error metrics. For vanilla Newton's method (\cref{secVanillaNewtonBoosting}), the relevant metric is the cosine angle between $f_{k+1}$ and $f^w_{k+1}$ in the induced Hessian norm, whereas for the globally convergent gradient regularized scheme (\cref{secGradientRegularizedNewton}), the metric takes the form of a weak learner edge condition between the true gradient and the implied weak gradient.

\subsection{GBDTs as Restricted Convex Optimization}\label{subsec:GBDTsAsRestrictedConvexOptimization}

In this section, we provide an overview of gradient boosting from a functional $\mathcal{L}^2(\nu)$ perspective \cite{2011GeneralizedBoostingAlgorithmsForConvexOptimization, 2020GeneralizedBoosting}, in the context of Restricted Newton Descent. Let $(X, Y) \sim \nu$ be a random sample with its corresponding probability measure $\nu$. Denote by $X \sim \nu_X$ the features in $\bbR^q$, with targets $Y \sim \nu_Y$ in $\bbR^K$, and $\nu_{Y|X=x}$ the conditional law. Let $\widehat{\nu}_N = \frac{1}{N}\sum_{i=1}^N \delta_{(x_i, y_i)}$ be the empirical measure of $\nu$ with respect to a training dataset $\{x_i, y_i\}_{i=1}^N$ of size $N$. We work in the Hilbert space $\mathcal{L}^2(\nu_X)$ of square integrable functions defined on the feature space, with inner product $\langle f, g\rangle = \E_{\nu_X}[\langle f(X),g(X) \rangle_{\bbR^K}]$. When $\nu = \widehat{\nu}_N$, this space can be identified with $\mathbb{R}^{N\times K}$ using the inner product $\langle f, h\rangle = \frac{1}{N}\sum_{i=1}^N \langle f(x_i),h(x_i) \rangle_{\bbR^K}$. Let $l$ be a twice differentiable convex loss function.

To derive the Newton boosting algorithm used by e.g. XGBoost and LightGBM as a special case of Restricted Newton Descent, we simply take $\mathcal{F}$ to be the set of decision trees up to a fixed depth, with base space $\mathcal{H} = \mathcal{L}^2(\widehat{\nu}_{X,N})$, and loss function $L(F) = \frac{1}{N} \sum_{i=1}^N l(F(x_i), y_i) = \E_{\widehat{\nu}_N}[l(F(X),Y)]$. In the space $\mathcal{L}^2(\nu_X)$, the Fréchet gradient and Hessian take the form
\begin{align}
    \nabla L (F)(x)  =  \E_{\nu_{Y|X=x}}\big[ \partial_{1} l(F(x), Y) \big],
\end{align}
\begin{align}
    \nabla^2 L (F)[f](x)  =  \E_{\nu_{Y|X=x}}\big[ \partial_{11} l(F(x), Y) f(x) \big],
\end{align}
\newpage
and are easily computed for empirical measures $\widehat{\nu}_N$ as partial derivatives and matrices with respect to $F(x_i)$. Write $g_{(i)} = \partial_1 l(F_k(x_i), y_i) \in \bbR^K$ and $h_{(i)} = \partial_{11} l(F_k(x_i), y_i) \in \bbR^{K\times K}$. Following the XGBoost derivation \cite{2016XGBoost}, the objective becomes to minimize
\begin{align*}
        L(F_k+f) 
        &= \frac{1}{N}\sum_{i=1}^N l(F_k(x_i)+f(x_i), y_i) 
        \approx  L(F_k) + \frac{1}{N}\sum_{i=1}^N \left\langle f(x_i), g_{(i)} + \frac{1}{2} h_{(i)} f(x_i) \right\rangle_{\mathbb{R}^K} 
\end{align*}
which we see coincides with the Hilbert space derivation using $Q_k$. Furthermore, in boosting we minimize over the set of weak learners $f \in \mathcal{F}$ rather than over all $\mathcal{L}^2(\widehat{\nu}_N)$. Letting $f$ be a decision tree of the form $f(x) = \sum_j w_j I_j(x)$ where $I_j(x) \in \{0, 1\}$ is a disjoint partition of the feature space, one finds that the optimal weights $w_j$ minimizing $Q_k(f)$ are
\(
    w_j =  - ( \sum_{i \in I_j} h_{(i)})^{-1}(\sum_{i \in I_j} g_{(i)}),
\)
coinciding with the classical XGBoost formula. See \cref{appendixXGBFormulaDerivation} for more details and full derivation.

\vspace{-5pt}
\section{Restricted Newton for Hessian-dominated Losses}\label{secVanillaNewtonBoosting}
\vspace{-5pt}

Before introducing the globally convergent scheme in \cref{secGradientRegularizedNewton}, we first analyze the vanilla Newton update via Restricted Newton Descent. We prove for smooth strongly convex losses satisfying a Hessian-dominance condition, that vanilla Newton's method with weak learners converges globally with a linear rate. This analysis is inspired by the LogitBoost work \cite{2014AConvergenceRateAnalysisForLogitBoostMARTAndTheirVariant}. However, our analysis uses a functional formulation on Hilbert spaces, which includes standard Newton’s method as a special case and accommodates losses beyond logistic loss, rather than the matrix-based formulation of the work on LogitBoost. We also introduce a more direct, mathematically natural, trajectory-dependent definition of weak learner cosine angles. To ensure the Newton update is well-defined on $\mathcal{H}$, we make the following assumption:


\begin{assumption}\label{assumptionHessianNewtonEquationUniqueSolution}
    The equation $g_k = -H_k f$ admits a unique solution $f=f_{k+1} = - H_k^{-1} g_k$.
\end{assumption}

\cref{assumptionHessianNewtonEquationUniqueSolution} is satisfied, for example, when $L$ is strongly convex, or in most practical settings when $\mathcal{H}$ is a finite dimensional space with a strictly convex loss $L$. Note that this assumption is not required for the gradient regularized scheme in \cref{secGradientRegularizedNewton}.

\subsection{Cosine Angle}

Recall that the iterates of Restricted Newton Descent are derived by minimizing $Q_k(f) = \left\langle f, g_k + \frac{1}{2}H_k f \right\rangle$ restricted to $f\in\mathcal{F}$ at each iteration $k$, yielding $F_k = \sum_{t=1}^k \eta f^w_t$. In order to analyze the restricted Newton scheme, we need a way to measure the error between the exact direction $f_{k+1}$ and the weak learner $f^w_{k+1}$. It turns out that the critical quantity determining the performance of this scheme, is the cosine angle between $f_{k+1}$ and $f^w_{k+1}$ measured in the Newton geometry induced by the Hessian. For this, we use the following norm definition.

\begin{definition}\label{defCosineAngle}
    Let $A : \mathcal{H} \to \mathcal{H}$ be a self-adjoint positive operator. We define the inner product induced by $A$ as
    \begin{equation}
        \langle h, f\rangle_{A} := \langle h, A f\rangle.
    \end{equation}
    When $A = H_k$, we call this the \textbf{Hessian-induced inner product}, with induced norm $\|\cdot\|_{H_k}$.
\end{definition}

The following lemma motivates the use of the Hessian-induced norm and naturally leads to the definition of the cosine angle between weak and exact iterates.

\begin{lemma}\label{lemmaVanillaNewtonClosestElement}
    Minimizing $Q_k(f)$ is equivalent to finding a closest element $f\in\mathcal{F}$ to the exact Newton direction $f_{k+1}$ in the Hessian-induced norm $\|\cdot\|_{H_k}$.
\end{lemma}
\begin{proof}
    Using \cref{defCosineAngle} and the exact Newton update $f_{k+1} = - H_k^{-1} g_k$, we can write
    \begin{align*}
        \|f-f_{k+1}\|^2_{H_k}
        = \|f\|_{H_k}^2 - 2\langle f, -g_k\rangle + \| f_{k+1} \|_{H_k}^2
        = 2Q_k(f) + \|f_{k+1}\|^2_{H_k}.  \quad\qquad \qedhere
    \end{align*}
\end{proof}

\begin{definition}\label{defHessianCosineAngle}
    Consider the iterates produced by vanilla Newton's method. At each iteration $k$, we define the Hessian-induced cosine angle $\Theta_k \in [0, 1]$ between the exact update $f_{k+1}$ and the weak learner $f_{k+1}^w$ as
    \begin{equation}\label{eqHessianCosineAngle}
        \Theta_k := \frac{\langle f_{k+1}, f_{k+1}^w \rangle_{H_k}}{\|f_{k+1}\|_{H_k} \|f_{k+1}^w\|_{H_k}}, \qquad \Theta := \inf_k \Theta_k.
    \end{equation}
\end{definition}
\newpage

In contrast to \cite{2014AConvergenceRateAnalysisForLogitBoostMARTAndTheirVariant}, which uses a weak learnability assumption specific to binary classification, more recent prior work uses definitions similar in spirit to ours, but defines them by taking suprema and infima over all admissible weak learner directions and target directions \cite{2020AcceleratingGradientBoostingMachines,2020RandomizedGradientBoostingMachine,2020SnapBoostAHeterogeneousBoostingMachine}; see also \cite{2011GeneralizedBoostingAlgorithmsForConvexOptimization}. Our definition is simpler and more direct, capturing the specific two vectors $f^w_{k+1}$ and $f_{k+1}$ determining the cosine angle. Using this definition, we immediately obtain the following exact algebraic properties connecting the weak and exact Newton iterates with the gradient $g_k$, directly via the cosine angle $\Theta_k$. See \cref{app:GlobalConvergenceProofs} for proof. These identities and related results provide precise control over the weak learner approximation errors, enabling us to adapt more traditional convex optimization techniques to establish convergence rates for boosting.

\begin{proposition}\label{lemmaCosineAngleProperties}
    If $\mathcal{F}$ is closed under scalar multiplication, then the weak and exact Newton directions satisfy
    \begin{enumerate}
        \item $\|f_{k+1} - f_{k+1}^w\|^2_{H_k} = (1-\Theta_k^2) \|f_{k+1}\|^2_{H_k}$,
        \item $\|f_{k+1}^w\|_{H_k}^2 \;=\; \Theta_k^2 \|f_{k+1}\|_{H_k}^2 \;=\; -\langle g_k, f_{k+1}^w\rangle$
    \end{enumerate}
\end{proposition}

\subsection{Decrease in Loss for Vanilla Restricted Newton}

The restricted Newton scheme guarantees a strict decrease in loss at each iteration, forming the basis of our convergence analysis. This relies on standard regularity assumptions used in convex optimization: $L$ being $S$-smooth and $\mu$-strongly convex, meaning the Hessian satisfies $\mu I \preceq \nabla^2 L(F) \preceq S I$. See \cref{app:SmoothnessLemmas} for more details of assumptions and standard inequalities that follow. Additionally, we introduce the following definition of a Hessian-dominated loss that will be useful in our analysis.

\begin{definition}\label{defHessianDominatedLoss}
    Assume $L_\star:=\inf_{F\in\mathcal H}L(F)>-\infty$. We say that $L$ is \textbf{Hessian-dominated} if there exists a constant $c>0$ such that the weak iterates $F_k$ and exact Newton direction $f_{k+1}$ satisfy
    \[
        \|f_{k+1}\|_{H_k}^2
        \geq c\bigl(L(F_k)-L_\star\bigr)
        \qquad\text{for all } k.
    \]
\end{definition}

 \cref{defHessianDominatedLoss} can be viewed as a Polyak--\L{}ojasiewicz condition \citep{2016LinearConvergenceOfGradientAndProximalGradientMethodsUnderThePolyakLojasiewiczCondition} stated in the local Hessian norm along the optimization trajectory. This definition also generalizes the condition derived for LogitBoost in \citep[Theorem 4]{2014AConvergenceRateAnalysisForLogitBoostMARTAndTheirVariant} to general losses using the language of Hessian norms and weak iterates.  Examples of Hessian-dominated losses include binary and categorical cross entropy, both with constant $c=1$, as proved in \cref{app:HessianDominatedLoss}. Their $\ell_2$-regularized versions are also Hessian-dominated;
see \cref{lem:regularized_loss_hessdom,lem:regularized_binary_ce_hessdom}. We are now ready to state our main convergence result for vanilla Restricted Newton Descent.

\begin{theorem}[\textbf{Global Rate I}]\label{lemmaStandardSmoothnessResults}
    If $L$ is $S$-smooth, $\mu$-strongly convex, and $\mathcal{F}$ closed under scalar multiplication, then for any learning rate $\eta \in (0, \frac{2\mu}{S})$ we have
    \begin{align}
        L(F_k) - L(F_{k+1}) 
        &\geq \eta\left( 1 -\frac{\eta S}{2\mu}\right) \Theta^2\|f_{k+1}\|_{H_k}^2.
    \end{align}
    Consequently, if $L$ is Hessian-dominated, then the optimality gap $L(F_k)-L_\star$ converges linearly with contraction factor $1-\rho$, where $\rho = c\Theta^2 \eta(1 - \frac{\eta S}{2\mu})$:
    \begin{align}
        L(F_{k+1})-L_\star \leq (1-\rho)\bigl(L(F_k)-L_\star\bigr).
    \end{align}
\end{theorem}
\begin{remark}
    Our bound compares favourably to SnapBoost \cite{2020SnapBoostAHeterogeneousBoostingMachine}, which proved a linear rate with contraction factor $1-\Theta_{\mathrm{snap}}^2\mu^2/S^2$ for the fixed learning rate $\eta=\mu/S$. With the same learning rate, our result gives the contraction factor $1-c\Theta^2\mu/(2S)$. Thus, when the Hessian-dominance constant $c$ is independent of the condition number $S/\mu$, our bound has better dependence. Furthermore, we achieve essentially the same rate as \citep{2014AConvergenceRateAnalysisForLogitBoostMARTAndTheirVariant}, which uses probability clamping and a change-of-Hessian lemma to establish a linear rate for logistic loss. Our theorem instead uses smoothness, strong convexity, and Hessian dominance.
\end{remark}

While cross entropy loss itself is not strongly convex, this issue is typically addressed in two ways. First, adding $\ell_2$-regularization to the loss function yields a strongly convex loss. Second, probability values can be clipped to $[\epsilon, 1-\epsilon]$ for some $\epsilon>0$, which implicitly implies strong convexity. An alternative approach to \cref{lemmaStandardSmoothnessResults} specific to trees is to clip the leaf values rather than probabilities, in combination with a \textit{change of Hessian lemma}, similar to \citep[Lemma 9]{2014AConvergenceRateAnalysisForLogitBoostMARTAndTheirVariant}. However, we believe our approach offers a more concise theoretical treatment.

\section{Gradient Regularized Restricted Newton}\label{secGradientRegularizedNewton}
Despite the fast quadratic local convergence of Newton's method in finite-dimensional convex optimization, the method does not have global convergence guarantees for general smooth convex functions \cite{2016SimpleExamplesForTheFailureOfNewtonsMethodWithLineSearchForStrictlyConvexMinimization}. A classical example of this is $L(x) = \sqrt{1+x^2}$ defined for $x\in \mathbb{R}$, see \cref{app:newton_divergence}. However, cubic regularized Newton \cite{2006CubicRegularizationOfNewtonMethodAndItsGlobalPerformance} and the more recent gradient regularized Newton \cite{2023RegularizedNewtonMethodWithGlobalConvergence,2024GradientRegularizationOfNewtonMethodWithBregmanDistances} achieve a global $O(\frac{1}{k^2})$ convergence rate. In this section, we extend the latter scheme to Newton boosting with weak learners within our framework of Restricted Newton Descent.

Instead of minimizing the local quadratic approximation $Q_k(f)$ at each iteration $k$, gradient regularized Newton's method minimizes the following regularized objective:
\begin{align*}
    Q_k^{reg}(f) 
    \; = \; 
    \left\langle f,g_k + \frac{1}{2}H_k[f]  \right\rangle + \frac{\lambda_k}{2}\|f\|^2
    \; = \; 
    Q_k(f) + \frac{\lambda_k}{2}\|f\|^2,
\end{align*}
where $\lambda_k \in (0, \infty)$ is a scalar regularization parameter. For losses with $2M$-Lipschitz continuous Hessians, we will allow values $\lambda_k \geq \sqrt{M\|g_k\|}$. Minimizing $Q_k^{reg}(f)$ yields the exact regularized update
\begin{align}\label{eqRegBoostingRoundExact}
    f_{k+1}  = - (H_k + \lambda_k I)^{-1}[g_k].
\end{align}

\begin{wrapfigure}{r}{0.58\textwidth}
  \vspace{-20pt} 
  \begin{minipage}{\linewidth} 
    \begin{algorithm}[H]
    \caption{Gradient Regularized Restricted Newton}
    \label{alg:boostingGradRegNewton}
    \begin{algorithmic}[1]
        \STATE \textbf{Input:} Loss $L:\mathcal{H}\to\bbR$ with $2M$-Lipschitz Hessian. Learning rate $\eta$. Weak learner family $\mathcal{F}$. Initial $F_0$.
        \FOR{$k = 0, 1, 2, \dots$}
            \STATE $g_k, \: H_k \leftarrow \nabla L(F_k), \: \nabla^2L(F_k)$
            \STATE Choose $\lambda_k \geq \sqrt{M\|g_k\|}$
            \STATE $\big($Exact target $f_{k+1} \leftarrow - (H_k + \lambda_k I)^{-1}g_k \big)$
            \STATE Weak learner $f^{w}_{k+1} \leftarrow \argmin_{f \in \mathcal{F}} Q_k^{reg}(f)$
            \STATE $F_{k+1} \leftarrow F_k + \eta f^{w}_{k+1}$
        \ENDFOR
    \end{algorithmic}
    \end{algorithm}
  \end{minipage}
  \vspace{-15pt} 
\end{wrapfigure}

In other words, gradient regularized Newton is simply an $\ell_2$-regularized Newton step where the $\ell_2$-regularization is adaptive at each iteration, being proportional to the square root of the gradient norm $\sqrt{\|g_k\|}$. Similar to \cref{secVanillaNewtonBoosting}, we consider Restricted Newton Descent using weak learners rather than the exact update. The full gradient regularized scheme is detailed in \cref{alg:boostingGradRegNewton}, with the implicit exact step denoted in parentheses.

\subsection{Convergence of Restricted GRN}

We impose a standard regularity condition on the loss $L$ in terms of the Lipschitz continuity of its Hessian, as is common for regularized Newton methods \cite{2006CubicRegularizationOfNewtonMethodAndItsGlobalPerformance, 2008AcceleratingTheCubicRegularizationOfNewtonsMethodOnConvexProblems, 2023RegularizedNewtonMethodWithGlobalConvergence, 2024GradientRegularizationOfNewtonMethodWithBregmanDistances}.

\begin{assumption}\label{assumptionHessLipAndCosine}
    The Hessian $\nabla^2 L : \mathcal{H} \to \mathcal{L(H)}$ is $2M$-Lipschitz continuous for some constant $M>0$:
    \vspace{-7pt}\begin{align*}
        \|\nabla^2L(f)-\nabla^2L(g)\| \leq 2M\|f-g\| \qquad \forall f,g\in \mathcal{H}.
    \end{align*}
\end{assumption}

The following proposition bounds the change in loss and guarantees strict decrease for $0<\eta<\sqrt{3}$ whenever the weak update is nonzero. From this point on, we assume that $\mathcal{F}$ is closed under scalar multiplication. All proofs for this section are deferred to \cref{app:GlobalConvergenceProofs}.

\begin{proposition}\label{lemmaDecreaseInLossNEW}
If $L$ is convex with $2M$-Lipschitz Hessian, with learning rate $\eta\in(0,2]$, then the iterates of \cref{alg:boostingGradRegNewton} satisfy 
\begin{align}
    L(F_k) - L(F_{k+1}) 
    &\geq \left( \eta - \frac{\eta^3}{3} \right) \lambda_k \|f^w_{k+1}\|^2.\label{eq:grn_loss_decrease}
\end{align}
\end{proposition}

The bound \eqref{eq:grn_loss_decrease} guarantees strict decrease in loss for $0<\eta<\sqrt{3}$ whenever $f_{k+1}^w\neq0$. For $\sqrt{3}\leq\eta<2$, the inequality remains valid but does not guarantee strict descent.

We now require a way to quantify how well the weak updates $f^w_{k+1}$ approximate the exact updates $f_{k+1}$ in the gradient regularized scheme. Previously in \cref{lemmaStandardSmoothnessResults} for vanilla Newton, the decrease in loss was determined by the Hessian-induced norm $\|\cdot\|_{H_k}$. In contrast, for the regularized scheme the decrease in loss is governed by the quantity $\lambda_k \|f^w_{k+1}\|$. As such, the cosine angle from \cref{defHessianCosineAngle} is no longer well suited for theoretical analysis of GRN, since the scheme and its convergence proof are fundamentally driven by gradient growth rather than directional alignment. Consequently, we introduce the notion of an implied weak gradient $g^w_k$, together with a corresponding weak gradient edge condition. This condition essentially states that the weak gradient remains within a relative radius of the true gradient, at each iteration $k$.

\begin{definition}\label{defWeakGradientEdge}
    We define the \textbf{implied weak gradient} as
    \begin{align}\label{eqWeakGradient}
        g_k^w : = - (H_k + \lambda_k I)f_{k+1}^w.
    \end{align}
    We say that the iterates $f^w_{k+1}$ have \textbf{weak gradient edge} $\gamma = \inf_k \gamma_k$ if there exists real numbers $0<\gamma_k \leq 1$ satisfying 
    \begin{align}\label{eqWeakGradientEdge}
        \|g_k^w - g_k\|^2 \leq (1-\gamma_k^2) \|g_k\|^2.
    \end{align}
\end{definition}

The following result relates the gradient at consecutive iterations and serves as the primary tool for our convergence analysis.

\begin{proposition}[\textbf{Gradient Growth}]\label{lemmagTProperties}
    If the iterates of $\cref{alg:boostingGradRegNewton}$ have weak gradient edge $\gamma$, with learning rate $\eta\in(0,1]$, then
    \vspace{-10pt}\begin{align}\label{eqGradientGrowthMainPaper}
        \|g_{k+1}\|
        &\leq \left(1 + \eta^2 + \eta \sqrt{1-\gamma^2}\right)\|g_k\|.
    \end{align}
\end{proposition}

We now present our main results: establishing the global and local convergence rates of gradient regularized Newton's method with weak learners. Despite the inexact updates inherent to boosting and Restricted Newton Descent, \cref{alg:boostingGradRegNewton} retains the classical $\mathcal{O}(\frac{1}{k^2})$ rate, matching first-order momentum boosting \cite{2020AcceleratingGradientBoostingMachines} without the requirement of fitting two trees per iteration. In comparison, standard first-order boosting without momentum
has a global $\mathcal{O}(1/k)$ rate. We use the following standard definition from the cubic Newton literature \cite{2006CubicRegularizationOfNewtonMethodAndItsGlobalPerformance,2023RegularizedNewtonMethodWithGlobalConvergence}.

\begin{definition}
    $L$ has a \textbf{bounded initial sublevel set} if a finite optimum $F^* \in \argmin_{F \in \mathcal{H}} L(F)$ exists, and if there exists a constant $\Lambda>0$ such that $L(F) \leq L(F_0) \implies \|F-F^*\| \leq \Lambda$.
\end{definition}

\begin{theorem}[\textbf{Global Rate II}]\label{thm:global_rate_II}
    Assume $L$ is convex with $2M$-Lipschitz Hessian, a bounded initial sublevel set with $\|F-F^*\|\leq\Lambda$, and weak gradient edge $\gamma$. Let $\lambda_k = C\sqrt{M\|g_k\|}$ for some $C\geq 1$. For any learning rate $\eta \in (0,1]$, the gradient regularized Newton scheme has global convergence rate
    \begin{align*}
        L(F_k) - L(F^*) \leq
        \mathcal{O}\left(\max\left\{\frac{C^2 M \Lambda^3}{\eta^2 k^2 \gamma^{12}},\:\: \Lambda\exp\left(-k\eta\gamma^2\big(\frac{\gamma^2}{2}+2-2\eta\big) / 16\right)\right\}\right).
    \end{align*}
\end{theorem}

The proof proceeds by analyzing two regimes based on the fraction of iterates where the gradient fails to contract sufficiently, determined by $\eta$ and $\gamma$. In the first regime, we derive a recurrence of the form $\alpha_{k+1} \leq \alpha_k - \tau \alpha_k^{3/2}$ for some $\tau > 0$, yielding a $\mathcal{O}(\frac{1}{\eta^2 \gamma^{12} k^2})$ rate. In the second regime, the cumulative contraction dominates, resulting in a linear rate. The global rate is governed by the maximum of these two bounds; see \cref{app:GlobalConvergenceProofs} for the full proof.

\begin{theorem}[\textbf{Local Rate}]\label{thm:local_rate}
Let $\lambda_k = C\sqrt{M\|g_k\|}$ for some $C\geq 1$, and let the learning rate be $\eta\in(0,1]$. If $L$ is $\mu$-strongly convex with $2M$-Lipschitz Hessian, weak gradient edge $\gamma$, and if there exists a $k_0\geq 0$ such that $\|g_{k_0}\| \leq \frac{\mu^2}{4C^2M}\left( \frac{1-\sqrt{1-\gamma^2}}{1+\sqrt{1-\gamma^2}}\right)^2$, then for all $k\geq k_0$ we have
\begin{align}
    \|g_{k+1}\|
    &\leq
    \left(
    1-\frac{\eta}{4}\bigl(1-\sqrt{1-\gamma^2}\bigr)
    \right)\|g_k\| \\
    &\leq
    \left( 1 - \frac{\eta \gamma^2}{8}\right) \|g_k\|.
\end{align}
Hence the iterates of \cref{alg:boostingGradRegNewton} converge locally linearly.
\end{theorem}

\begin{remark}[On Greedy Tree Construction]
The results in \cref{secVanillaNewtonBoosting,secGradientRegularizedNewton}
are stated, for simplicity, for the same weak learner family
$\mathcal{F}$ at each iteration $k$. In practice, libraries such as
XGBoost, LightGBM, and CatBoost use greedy decision tree construction rather
than finding a globally optimal tree for the quadratic minimization. Our proofs require only
optimality of the selected weak update along its scalar multiples,
rather than global optimality over the full tree class. Thus, the
results also apply to iteration-dependent families selected by
greedy tree construction, provided the leaf values exactly minimize the
corresponding quadratic objective for the chosen structure. The quantities $\Theta_k$ and $\gamma_k$ capture the approximation
quality of the resulting updates for vanilla and gradient
regularized Newton, respectively, including the effect of
greedy split finding.
\end{remark}

\section{Practical Considerations}

In this section, we discuss when the regularity assumptions of \cref{lemmaStandardSmoothnessResults,thm:global_rate_II,thm:local_rate} are satisfied in practice for boosting, and how the GRN scheme can be implemented into existing tree boosting algorithms with negligible computational overhead.

\subsection{Regularity of the Boosting Objective}

Our convergence analysis relies on standard convex optimization assumptions, such as Lipschitz-continuous Hessians, smoothness, and strong convexity. Using the same notation as \cref{subsec:GBDTsAsRestrictedConvexOptimization}, because the boosting empirical risk $L(F) = \frac{1}{N} \sum_{i=1}^N l(F(x_i), y_i)$ is constructed from a pre-specified element-wise loss $l$ (e.g., MSE, cross entropy, Charbonnier), one can prove that the regularity of $l$ carries over to $L$ in the boosting space $\mathcal{H} = \mathcal{L}^2(\widehat{\nu}_{X,N})$. We formalize this below in \cref{prop:regularity}. This result ensures that if the base loss $l$ is well-behaved, the boosting objective is as well. Moreover, because $l$ is chosen prior to training, its Lipschitz constant $M_0$ can be computed analytically. 

\begin{proposition}\label{prop:regularity}
    Assume that the inputs $x_1,\dots,x_N$ are pairwise distinct.
    If the element-wise loss $l(\cdot, y_i)$ has an $M_0$-Lipschitz Hessian in $\mathbb{R}^K$ uniformly across all fixed targets $y_i$, then the empirical risk $L$ has an $M_0\sqrt{N}$-Lipschitz Hessian in the boosting space $\mathcal{L}^2(\widehat{\nu}_{X,N})$. Additionally, if $l$ is $S$-smooth, $\mu$-strongly convex, or Hessian-dominated with constant $c$, then $L$ also is with the same constant values.
\end{proposition}

\subsection{Implementing GRN for GBDTs}
The GRN scheme is very simple to implement in practice and only differs from vanilla Newton boosting by an extra additive $\ell_2$-regularization term based on the gradient norm. Both methods use the exact same split finding, gain calculation, and leaf fitting formulas. The only difference lies in the effective $\ell_2$-regularization used. Standard tree boosting implementations like XGBoost and LightGBM use static $\ell_2$-regularization in the tree-fitting objective. GRN, on the other hand, uses an adaptive regularization term $\lambda_k$ given by
\[
    \lambda_k = \sqrt{M\|g_k\|}.
\]
An additional static term $\lambda_{\text{base}}$ can be added in practice, giving $\lambda_k = \lambda_{\text{base}} + \sqrt{M\|g_k\|}$, but our convergence guarantees are stated without this additional term. Using the notation of \cref{subsec:GBDTsAsRestrictedConvexOptimization}, the exact gradient norm in the boosting space is computed from the sample gradients $g_{(i)}$ as $\|g_k\| = \big(\frac{1}{N}\sum_{i=1}^N \|g_{(i)}\|^2_{\mathbb{R}^K}\big)^{1/2}$. Because updating this adaptive $\ell_2$ term only requires a single vector norm calculation, and the main computational bottleneck for GBDT algorithms is greedy split finding, the computational overhead of GRN is negligible. Note, however, that our code implementation scales the regularization added to each aggregated leaf Hessian by the number of samples in that leaf,
as suggested by our theory and derivations;
see \cref{remark:regularization_conventions} for more details.

\section{Numerical Experiments}\label{secNumericalResults}
In this section, we empirically validate our theoretical findings regarding the convergence of boosting schemes and the behaviour of weak learner edges. The purpose of this section is not to prove or disprove \textit{state of the art} predictive performance of the GRN boosting scheme, but rather to supplement the theory presented in earlier sections regarding training and convergence. We implement first-order gradient boosting, Newton boosting, and the Gradient Regularized Newton (GRN) boosting scheme in JAX \cite{2018jaxGithub}, following the standard depth-wise greedy tree construction strategy used by XGBoost \cite{2016XGBoost}. The exact code used to produce these experiments is included in the
supplementary material; an updated repository is available at
\url{https://github.com/nikitazozoulenko/global-newton-boosting}.

\subsection{Charbonnier loss divergence and convergence}
The Charbonnier loss, defined as $l(\widehat{y}, y) = \sqrt{1+(y-\widehat{y})^2}$ $-1$, is a classical example of a loss function where Newton's method can diverge (see \cref{app:newton_divergence}). We investigate if this also is true for boosting, using the Wine Quality UCI dataset \cite{UCIWineQuality} with depth 4 trees. 
For GRN we set $M=0.1$, and otherwise use a static $\ell_2$-regularization of $1.0$ when explicitly stated. Here these terms were tuned empirically rather than chosen from the global Hessian-Lipschitz bound in \cref{prop:regularity}.
\Cref{fig:Charbonnier} compares the training loss of Newton boosting, first-order gradient boosting, and GRN for different learning rates $\eta$. For $\eta=1$, vanilla Newton boosting diverges, whereas GRN decreases the fastest. 
GRN initially tracks the trajectory of fixed $\ell_2$-regularized Newton boosting but eventually outperforms it once the gradient norm decreases sufficiently. For the smaller value $\eta=0.1$, Vanilla Newton boosting minimizes the loss most rapidly in this scenario, but is only able to converge in this small learning rate setting.

\subsection{Weak learner conditions and tree depth}

To better understand the theoretical constants governing convergence, we train models on a subset of the Higgs dataset \cite{UCIHiggs} with $10000$ data points using BCE loss with $\ell_2$-regularization $\lambda_{\text{base}}=0.1$ for non-GRN models, and $M=0.01$ for GRN, giving roughly equivalent regularization at initialization. We use a learning rate $\eta=0.1$ for all models. We compute the cosine angle $\Theta_k$ and weak learner edge $\gamma_k$ in \cref{fig:bceTreeDepth} at every iteration $k$, via \eqref{eqHessianCosineAngle} and \eqref{eqWeakGradientEdge}, by instantiating $f^w_{k+1}, f_{k+1}, g^w_k$, and $g_k$ as vectors in $\mathbb{R}^N$. As expected, increasing tree depth results in stronger weak learners, reflected in higher values for $\Theta_k$ and $\gamma_k$. We observe that these metrics are high during the initial iterations of boosting and eventually plateau at positive values. As optimization progresses, fitting the diminishing residuals becomes increasingly difficult for the weak learners, naturally leading to lower weak learner edges. The trajectories of the cosine angle $\Theta_k$ and weak gradient edge $\gamma_k$ are remarkably similar, which is somewhat unexpected given that they are defined in different geometries. Conceptually, $\gamma_k$ (\cref{defWeakGradientEdge}) can be viewed as an approximate $(H_k + \lambda_k I)^2$-induced cosine angle between $f^w_{k+1}$ and $f_{k+1}$, in contrast to the standard $H_k$-induced cosine angle $\Theta_k$. 
If the error $f_{k+1}-f^w_{k+1}$ were orthogonal to $f^w_{k+1}$ in the $(H_k+\lambda_k I)^2$-induced inner product, then $\gamma_k$ would equal the cosine between the weak and exact updates in that inner product. However, this orthogonality is not guaranteed.

\begin{figure}[t]
  \vskip -0.0in
  \begin{center}
    \centerline{\includegraphics[width=0.9\columnwidth]{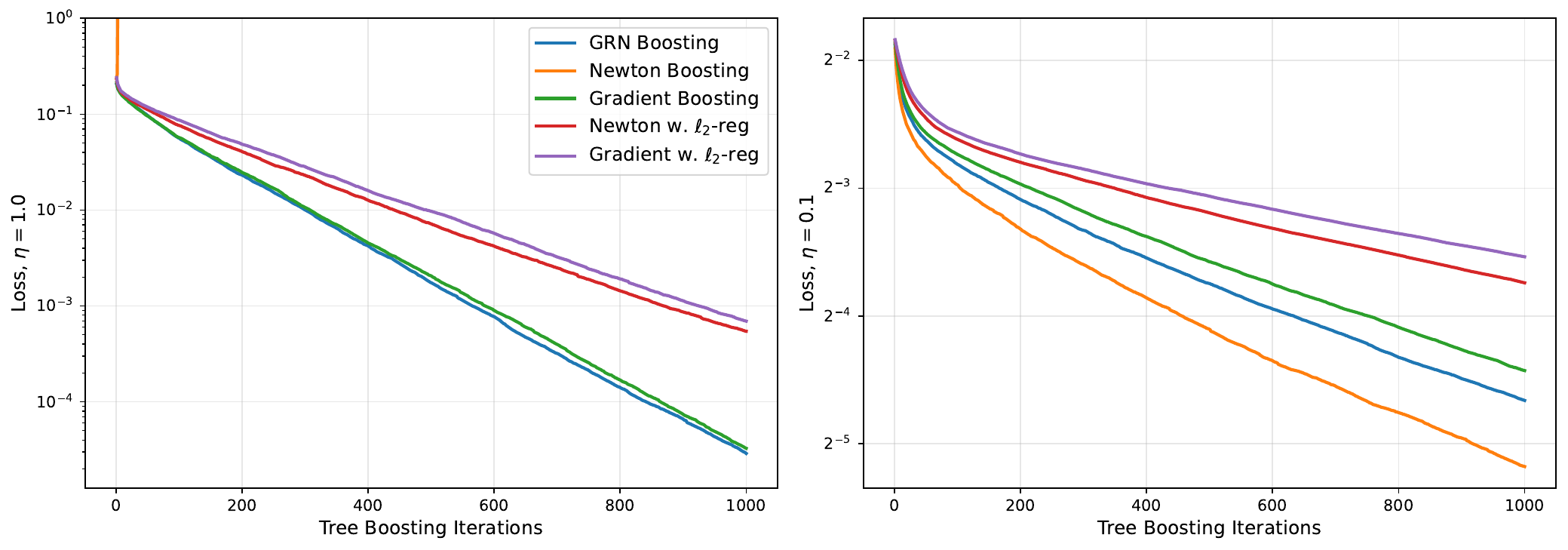}}
    \caption{
      Charbonnier train loss on the Wine Quality UCI dataset, for different gradient boosting schemes and learning rates $\eta$. Vanilla Newton boosting diverges for $\eta=1$.
    }
    \label{fig:Charbonnier}
  \end{center}
  \vskip 0.1in
\end{figure}

\begin{figure}[t]
  \vskip -0.0in
  \begin{center}
    \centerline{\includegraphics[width=1.0\columnwidth]{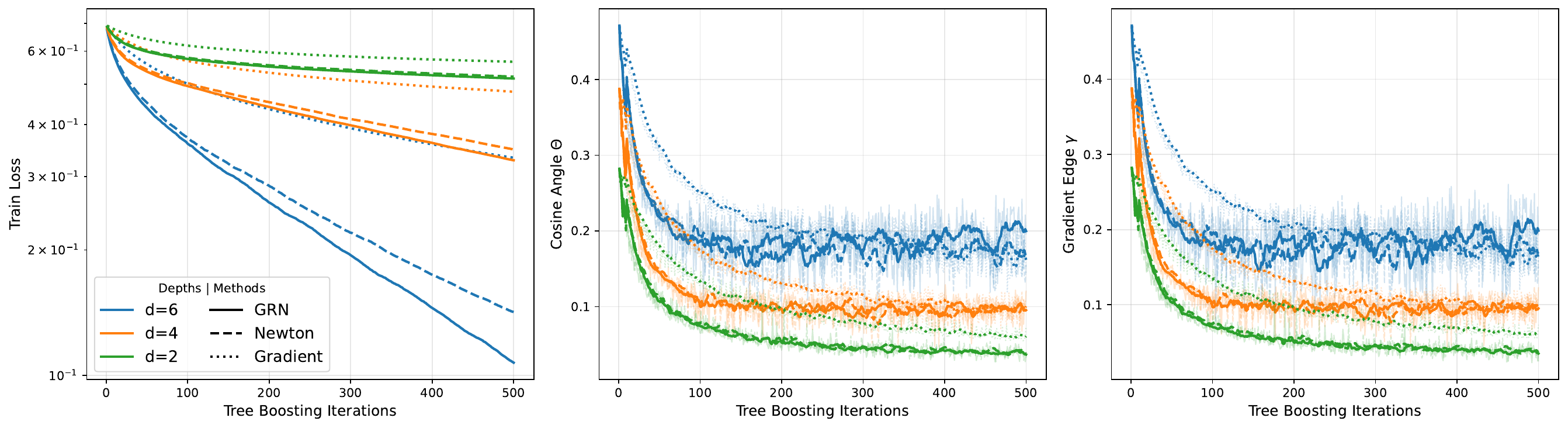}}
    \caption{
      Binary cross entropy, cosine angle $\Theta_k$, and weak gradient edge $\gamma_k$ computed at each boosting iteration $k$, for varying tree depths $d$. The exact values are noisy between iterations, but the rolling window average (highlighted in bold) remains stable.
    }
    \label{fig:bceTreeDepth}
  \end{center}
  \vskip 0.0in
\end{figure}

\newpage
\section*{Conclusion, Limitations, and Future Work}

By posing boosting as convex optimization, we introduced the Restricted Newton Descent framework to establish convergence rates for general weak learners, recovering Newton decision tree boosting as a special case. Using the cosine angle $\Theta_k$, we proved that vanilla Newton boosting achieves a linear convergence rate for Hessian-dominated losses. For general convex losses with Lipschitz Hessians, we analyzed a novel gradient regularized Newton (GRN) boosting scheme, establishing the first global $\mathcal{O}(1/k^2)$ convergence rate for second-order GBDTs, via the weak gradient edge $\gamma$. We also demonstrated that vanilla Newton boosting may diverge on such losses while the GRN scheme converges. 

However, there are some limitations to this work. The global theoretical guarantee comes with a heavy $\gamma^{12}$ worst-case dependence in the bound. This dependency is likely an intrinsic artifact of the GRN scheme being fundamentally based on Euclidean gradient growth, rather than relying solely on the local Hessian geometry. While this does not cause any problems in the classical exact theory \cite{2023RegularizedNewtonMethodWithGlobalConvergence} (corresponding to $\gamma=1$), the additional error terms introduced by the weak learners are difficult to control in this setting. Furthermore, because our work is primarily theoretical and aims to bridge the gap between GBDTs and traditional convex optimization, direct performance benchmarking against highly optimized C++ libraries falls outside our scope. Finally, while global Newton schemes are necessary to obtain global convergence guarantees, the empirical difference between the vanilla and GRN schemes can become minor at small learning rates.

Our results suggest several promising directions for future research. First, given that finite-dimensional cubic and gradient regularized Newton can be accelerated to $\mathcal{O}(1/k^3)$, investigating whether such acceleration can be realized in the boosting setting is a natural next step. Second, it would be interesting to study the Restriction operation for other optimization schemes to determine which convergence properties are retained under weak updates, and see if the dependence on $\gamma^{12}$ can be improved. Finally, investigating the possibility of superlinear rates with weak learners and establishing generalization bounds for Newton boosting remain open problems. Future work will also explore other applications of Restricted Newton Descent beyond the gradient boosting setting.

\section*{Acknowledgements}
NZ has been supported by the Roth Scholarship at Imperial College London. DF was supported by the Göran Gustafsson Foundation for Research in Natural Sciences and Medicine via Alexander Berglund's 2024 Göran Gustafsson Prize. The work of TC was supported in part by UK Research and Innovation (UKRI) through the Engineering and Physical Sciences Research Council (EPSRC) via Programme Grants [Grant No. UKRI1010: High order mathematical and computational infrastructure for streamed data that enhance contemporary generative and large language models]. We acknowledge computational resources and support provided by the Imperial College Research Computing Service (DOI: \texttt{10.14469/hpc/2232}). For the purpose of open access, the authors have applied a Creative Commons Attribution (CC BY) licence to any Author Accepted Manuscript version arising.

\bibliography{references}
\bibliographystyle{plainnat}

\newpage
\appendix
\onecolumn

\section{Discussions on Losses and Gradient Boosting Trees}

In this appendix, we provide explicit calculations for the gradients and Hessians of common loss functions used in boosting. We also provide the full derivation of the optimal weight update for decision trees within the Restricted Newton Descent framework, recovering the standard XGBoost formula.

\subsection{Common loss functions used in boosting}

We denote the prediction for the $i$-th data point as $u_i = F(x_i)$. The gradient and Hessian with respect to the prediction are denoted by $g_{(i)} \in \mathbb{R}^K$ and $h_{(i)} \in \mathbb{R}^{K \times K}$, respectively.

\textbf{Mean Squared Error (MSE):}
For the quadratic loss $l(\widehat{y}, y) = \frac{1}{2}(\widehat{y}-y)^2$, the gradient and Hessian are given by:
\begin{align*}
    g_{(i)} &= \widehat{y}_i - y_i, \\
    h_{(i)} &= 1.
\end{align*}
\textbf{Binary Cross Entropy (BCE):}
For binary classification, let $y_i \in \{0,1\}$ be the label and $u_i \in \mathbb{R}$ be the logit. The loss is defined as $l(u_i, y_i) = -[y_i \log \sigma(u_i) + (1-y_i) \log(1-\sigma(u_i))]$, where $\sigma(z) = (1+e^{-z})^{-1}$ is the sigmoid function. Using the chain rule, one finds that the derivatives are
\begin{align*}
    g_{(i)} &= \sigma(u_i) - y_i, \\
    h_{(i)} &= \sigma(u_i)(1-\sigma(u_i)).
\end{align*}
\textbf{Categorical Cross Entropy (CCE):}
In the multi-class setting ($K > 2$), let $y_i \in \{0,1\}^K$ be the one-hot encoded label and $u_i \in \mathbb{R}^K$ be the vector of logits. The loss is $l(u_i, y_i) = -\sum_{k=1}^K (y_i)_k \log (p_i)_k$, where $p_i = \operatorname{softmax}(u_i)$ is the vector of class probabilities. The gradient and Hessian are
\begin{align*}
    g_{(i)} &= p_i - y_i, \\
    h_{(i)} &= \text{diag}(p_i) - p_i p_i^\top,
\end{align*}
where $\text{diag}(p_i)$ is the diagonal matrix with elements of $p_i$ on the diagonal, and $p_i p_i^\top$ denotes the vector outer product.

\subsection{Newton Tree Boosting Derivation}\label{appendixXGBFormulaDerivation}

Let $f(x)$ be a decision tree defined by $f(x) = \sum_j w_j I_j(x)$, where $\{I_j\}_{j=1}^L$ is a disjoint partition of the feature space (i.e., the leaves of the tree) and $w_j \in \mathbb{R}^K$ are the leaf values. The second-order approximation of the loss, $Q(f)$, is given by
\begin{align*}
    Q(f) 
    &=\frac{1}{N} \sum_{i=1}^N \left\langle f(x_i), g_{(i)} + \frac{1}{2} h_{(i)} f(x_i) \right\rangle_{\mathbb{R}^K} \\
    &= \frac{1}{N} \sum_{i=1}^N \left\langle f(x_i), g_{(i)} \right\rangle_{\mathbb{R}^K} + \frac{1}{2N} \sum_{i=1}^N \left\langle f(x_i), h_{(i)} f(x_i) \right\rangle_{\mathbb{R}^K}.
\end{align*}
We analyze the linear and quadratic terms separately. Substituting the definition of the tree $f(x)$, we find that
\begin{align*}
    \sum_{i=1}^N \left\langle f(x_i), g_{(i)} \right\rangle_{\mathbb{R}^K} 
    &= \sum_{i=1}^N \left\langle \sum_j w_j I_j(x_i), g_{(i)} \right\rangle_{\mathbb{R}^K}  \\
    &= \sum_j \left\langle w_j, \sum_{i=1}^N I_j(x_i) g_{(i)}\right\rangle_{\mathbb{R}^K} \\
    &= \sum_j \left\langle w_j, \sum_{i \in I_j} g_{(i)}\right\rangle_{\mathbb{R}^K}.
\end{align*}
For the quadratic term, we use the property that the regions are disjoint, implying $I_j(x)I_k(x) = \delta_{jk}I_j(x)$, where $\delta_{jk}$ is the Kronecker delta. 
\begin{align*}
    \sum_{i=1}^N \left\langle f(x_i), h_{(i)} f(x_i) \right\rangle_{\mathbb{R}^K} 
    &= \sum_{i=1}^N \left\langle \sum_j w_j I_j(x_i), h_{(i)} \sum_k w_k I_k(x_i) \right\rangle_{\mathbb{R}^K} \\
    &= \sum_{i=1}^N \sum_j \sum_k I_j(x_i) I_k(x_i) \left\langle w_j, h_{(i)} w_k \right\rangle_{\mathbb{R}^K} \\
    &= \sum_{i=1}^N \sum_j I_j(x_i) \left\langle w_j, h_{(i)} w_j \right\rangle_{\mathbb{R}^K} \\
    &= \sum_j \sum_{i \in I_j} \left\langle w_j, h_{(i)} w_j \right\rangle_{\mathbb{R}^K} \\
    &= \sum_j \left\langle w_j, \left( \sum_{i \in I_j} h_{(i)} \right) w_j \right\rangle_{\mathbb{R}^K}.
\end{align*}
The total objective decomposes into a sum of independent objectives for each leaf $j$. Differentiating with respect to a specific weight $w_j$, we obtain the local quadratic
\begin{align*}
    q(w_j) =  \left\langle w_j, \sum_{i \in I_j} g_{(i)} \right\rangle 
    + \frac{1}{2} \left\langle w_j, \left( \sum_{i \in I_j} h_{(i)} \right) w_j \right\rangle.
\end{align*}
Taking the gradient with respect to $w_j$ yields
\begin{align*}
    \nabla q(w_j)
    &= \sum_{i \in I_j} g_{(i)} + \left( \sum_{i \in I_j} h_{(i)} \right) w_j.
\end{align*}
Setting the gradient to zero gives the optimal weight update
\begin{align*}
    w_j^* &=  -\left( \sum_{i \in I_j} h_{(i)} \right)^{-1} \sum_{i \in I_j} g_{(i)}.
\end{align*}
Substituting $w_j^*$ back into the objective (and defining $G_j = \sum_{i \in I_j} g_{(i)}$ and $H_j = \sum_{i \in I_j} h_{(i)}$), we recover the standard score used in Newton boosting and XGBoost:
\begin{align*}
    Q(f^*) = \frac{1}{N} \sum_j q(w_j^*) 
    &=  -\frac{1}{2N} \sum_j \left\langle H_j^{-1} G_j , G_j \right\rangle_{\mathbb{R}^K}.
\end{align*}

\begin{remark}\label{remark:regularization_conventions}
The $\ell_2$-regularization used in standard implementations
like XGBoost and LightGBM differs slightly from the one suggested by our theory.
These libraries add a fixed scalar $\lambda$ to the sum
of sample Hessians within each leaf, yielding the update
$$w_j^* = -(H_j + \lambda I)^{-1} G_j.$$
By contrast, our framework regularizes the weak learner in the
Hilbert space $L^2(\widehat{\nu}_N)$.
In this space, the regularization modifies the Hessian operator
in the quadratic model pointwise,
which is equivalent to replacing each $h_{(i)}$ with
$h_{(i)}+\lambda I$, giving the update
\[
    w_j^*=-(H_j+|I_j|\lambda I)^{-1}G_j.
\]
Thus, in our formulation of boosting the theoretically consistent regularization scales
proportionally with the number of samples routed to each leaf. From here, there are two ways to proceed.
One can regularize only the quadratic step,
$Q_k(f)+\frac{\lambda}{2}\|f\|_{\mathcal H}^2$,
or instead regularize the loss itself,
$L_\mu(F)=L(F)+\frac{\mu}{2}\|F\|_{\mathcal H}^2$.
The latter modifies both the sample gradients and Hessians:
\[
    g_{(i)}^\mu=g_{(i)}+\mu F_k(x_i),
    \qquad
    h_{(i)}^\mu=h_{(i)}+\mu I.
\]
If the GRN boosting scheme from \cref{secGradientRegularizedNewton} is applied to this regularized loss, its adaptive
regularization must also use the modified gradient
$g_k^\mu=g_k+\mu F_k$, giving
$\lambda_k=\sqrt{M\|g_k^\mu\|}$. Our numerical experiments in \cref{secNumericalResults}
use the former approach: we regularize only the quadratic step,
as is standard in boosting libraries, but with the leaf-size
scaling described above.
However, if regularization is intended to make the loss
strongly convex in the space $L^2(\widehat{\nu}_N)$, it must instead be applied to the loss itself.
\end{remark}

\clearpage

\section{Some Classical Optimization Results}\label{app:SmoothnessLemmas}

We collect several standard smoothness and Lipschitz-Hessian definitions and results from optimization theory \cite{2006CubicRegularizationOfNewtonMethodAndItsGlobalPerformance,2018LecturesOnConvexOptimization}, and present them in the Hilbert space setting.

\begin{definition}
    A function $L : \mathcal{H} \to \bbR$ is said to be \textit{Fréchet differentiable} if there exists a mapping $\nabla L : \mathcal{H} \to \mathcal{H}$ such that for all $h,f \in \mathcal{H}$
    \begin{align*}
        L(h + f) = L(h) + \langle f, \nabla L (h)\rangle + o(\|f\|).
    \end{align*}
    The element $\nabla L(h)$ is called the \textit{gradient} at $h$. A function $L : \mathcal{H} \to \bbR$ is said to be \textit{twice Fréchet differentiable} if $L$ is Fréchet differentiable, and if for all $h\in\mathcal{H}$ there exists a bounded, self-adjoint linear operator $\nabla^2 L(h) : \mathcal{H} \to \mathcal{H}$ such that
    \begin{align*}
    L(h + f) = L(h) + \left\langle f, \nabla L(h) + \frac{1}{2}  \nabla^2 L(h)[f]
   \right\rangle + o(\|f\|^2),
    \end{align*}
    for all $f \in \mathcal{H}$. The map $\nabla^2 L(h)$ is called the \textit{Fréchet Hessian} (or second derivative) of $L$ at $h$.
\end{definition}

\begin{definition}[$S$-smooth loss]
Let $\mathcal H$ be a real Hilbert space. A Fr\'echet differentiable function
$L : \mathcal H \to \mathbb R$ is called \emph{$S$-smooth} if its gradient is Lipschitz continuous
with constant $S$, i.e.,
\[
\|\nabla L(f) - \nabla L(g)\| \le S \|f - g\|
\quad \forall f,g \in \mathcal H.
\]
Equivalently, if $L$ is twice Fr\'echet differentiable, then its Hessian satisfies
\[
\|\nabla^2 L(h)\|_{\mathrm{op}} \le S
\quad \forall h \in \mathcal H,
\]
For convex $L$, this is equivalent to $\nabla^2 L(h)\preceq SI$.
\end{definition}

\begin{definition}[$\mu$-strongly convex loss]
A function $L : \mathcal H \to \mathbb R$ is \emph{$\mu$-strongly convex} for some $\mu>0$ if one
of the following equivalent conditions holds:
\begin{enumerate}
    \item \textbf{(Quadratic lower bound):}
    The function $h \mapsto L(h) - \frac{\mu}{2}\|h\|^2$ is convex.
    \item \textbf{(First-order condition):}
    For all $f,g \in \mathcal H$,
    \[
    L(f) \ge L(g) + \langle \nabla L(g), f-g\rangle
    + \frac{\mu}{2}\|f-g\|^2.
    \]
    \item \textbf{(Gradient monotonicity):}
    For all $f,g \in \mathcal H$,
    \[
    \langle \nabla L(f) - \nabla L(g), f-g\rangle \ge \mu \|f-g\|^2.
    \]
    \item \textbf{(Second-order condition):}
    If $L$ is twice Fr\'echet differentiable, then
    \[
    \nabla^2 L(h) \succeq \mu I
    \quad \forall h \in \mathcal H.
    \]
\end{enumerate}
\end{definition}

\begin{definition}
The Fr\'echet Hessian $\nabla^2 L : \mathcal{H} \to \mathcal{L(H)}$ is $2M$-Lipschitz continuous if there exists a constant $M>0$ such that
    \begin{align*}
        \|\nabla^2L(f)-\nabla^2L(g)\| \leq 2M\|f-g\| \qquad \forall f,g\in \mathcal{H}.
    \end{align*}
\end{definition}

\begin{lemma}\label{lemmaStandardSmoothnessUpperBound}
    If the loss $L$ is $S$-smooth, then
    \begin{align}
        L(F_k+f) \leq L(F_k) + \langle g_k, f\rangle + \frac{S}{2} \|f\|^2.
    \end{align}
\end{lemma}
\begin{proof}
    Using Taylor's theorem with integral remainder, we express the expansion of the loss as
    \begin{equation}
        L(F_k+f) = L(F_k) + \langle g_k, f\rangle + \int_0^1 (1-t) \langle f, \nabla^2 L(F_k + tf) f \rangle \, dt.
    \end{equation}
    The definition of $S$-smoothness implies that the Hessian is bounded in operator norm by $S$, i.e., $\nabla^2 L(\cdot) \preceq S I$. Consequently, the quadratic form is bounded by $\langle f, \nabla^2 L(\cdot) f \rangle \leq S \|f\|^2$. Substituting this into the integral yields $\int_0^1 (1-t) S \|f\|^2 \, dt = \frac{S}{2}\|f\|^2$, completing the proof.
\end{proof}

\begin{lemma}\label{lemmaStandardHessianResults}
    If $\nabla^2 L$ is $2M$-Lipschitz, then for all $f,g\in \mathcal{H}$
    \begin{enumerate}
        \item $L(f) \leq L(g) + \langle \nabla L(g), f-g\rangle + \frac{1}{2}\langle \nabla^2 L(g)(f-g), f-g\rangle + \frac{M}{3}\|f-g\|^3$
        \item  $\|\nabla L(f) - \nabla L(g) - \nabla^2 L(g)(f-g)\| \leq M\|g-f\|^2$.
    \end{enumerate}
\end{lemma}
\begin{proof}
Let $h = f-g$. We begin with the Fundamental Theorem of Calculus and expand the terms:
\begin{align*}
L(f) - L(g) - \langle \nabla L(g), h \rangle
&= \int_0^1 \langle \nabla L(g+th) - \nabla L(g), h \rangle \,dt \\
&= \int_0^1 \left\langle \int_0^t \nabla^2 L(g+sh)[h]\,ds, h \right\rangle \,dt \\
&= \int_0^1 \int_0^t \langle h, \nabla^2 L(g+sh)[h] \rangle \,ds\,dt \\
&= \int_0^1 \int_0^t \langle h, \nabla^2 L(g)[h] \rangle \,ds\,dt + \int_0^1 \int_0^t \langle h, (\nabla^2 L(g+sh) - \nabla^2 L(g))[h] \rangle \,ds\,dt \\
&= \frac{1}{2} \langle h, \nabla^2 L(g)[h] \rangle + \int_0^1 \int_0^t \langle h, (\nabla^2 L(g+sh) - \nabla^2 L(g))[h] \rangle \,ds\,dt \\
&\leq \frac{1}{2} \langle h, \nabla^2 L(g)[h] \rangle + \int_0^1 \int_0^t \|h\| \cdot \|(\nabla^2 L(g+sh) - \nabla^2 L(g))[h]\| \,ds\,dt \\
&\leq \frac{1}{2} \langle h, \nabla^2 L(g)[h] \rangle + \int_0^1 \int_0^t \|\nabla^2 L(g+sh) - \nabla^2 L(g)\|_{op} \|h\|^2 \,ds\,dt \\
&\leq \frac{1}{2} \langle h, \nabla^2 L(g)[h] \rangle + \int_0^1 \int_0^t (2M s \|h\|) \|h\|^2 \,ds\,dt \\
&= \frac{1}{2} \langle h, \nabla^2 L(g)[h] \rangle + 2M\|h\|^3 \int_0^1 \frac{t^2}{2} \,dt \\
&= \frac{1}{2} \langle h, \nabla^2 L(g)[h] \rangle + \frac{M}{3}\|h\|^3.
\end{align*}
Rearranging the terms yields the desired first inequality. For inequality 2, we express the gradient difference as an integral of the Hessian.
\begin{align*}
    \|\nabla L(f) - \nabla L(g) - \nabla^2 L(g)[h]\|
    &= \left\| \int_0^1 \nabla^2 L(g+th)[h] \,dt - \int_0^1 \nabla^2 L(g)[h] \,dt \right\| \\
    &= \left\| \int_0^1 (\nabla^2 L(g+th) - \nabla^2 L(g))[h] \,dt \right\| \\
    &\leq \int_0^1 \| (\nabla^2 L(g+th) - \nabla^2 L(g))[h] \| \,dt \\
    &\leq \int_0^1 \| \nabla^2 L(g+th) - \nabla^2 L(g) \|_{op} \|h\| \,dt \\
    &\leq \int_0^1 (2M t \|h\|) \|h\| \,dt \\
    &= 2M \|h\|^2 \int_0^1 t \,dt \\
    &= M \|h\|^2.
\end{align*}
This completes the proof.
\end{proof}

\clearpage
\section{Convergence Rate Proofs and Related Technical Lemmas}\label{app:GlobalConvergenceProofs}

In this Appendix, we gather all proofs of our theoretical results presented in the main paper body.
\subsection{Vanilla Restricted Newton Proofs}

\begin{proof}[Proof of \cref{lemmaCosineAngleProperties}]
    Since $\mathcal{F}$ is closed under scalar multiplication, the objective $Q_k(\alpha f_{k+1}^w)$ is minimized at $\alpha=1$. Differentiating the equivalent objective (see \cref{lemmaVanillaNewtonClosestElement}) $\phi(\alpha) := \frac{1}{2}\|\alpha f_{k+1}^w - f_{k+1}\|_{H_k}^2$ with respect to $\alpha$ gives the derivative $\phi'(\alpha) = \langle f_{k+1}^w, \alpha f_{k+1}^w - f_{k+1} \rangle_{H_k}$. Setting this to zero at $\alpha=1$ yields
    $
        \langle f_{k+1}, f_{k+1}^w \rangle_{H_k} = \|f_{k+1}^w\|_{H_k}^2.
    $
    Substituting this into the definition of $\Theta_k$, we obtain
    \begin{equation}
        \Theta_k = \frac{\|f_{k+1}^w\|_{H_k}^2}{\|f_{k+1}\|_{H_k} \|f_{k+1}^w\|_{H_k}} = \frac{\|f_{k+1}^w\|_{H_k}}{\|f_{k+1}\|_{H_k}},
    \end{equation}
    which gives the first equality in the second item. The first item then follows immediately from expanding the norm using inner products. The remaining equality in the second item follows from
    $-\langle g_k, f_{k+1}^w\rangle
    = \langle H_k f_{k+1}, f_{k+1}^w\rangle
    = \langle f_{k+1}, f_{k+1}^w\rangle_{H_k}$.
\end{proof}

\begin{proof}[Proof of \cref{lemmaStandardSmoothnessResults}]
    Starting with the standard smoothness upper bound $L(F+f) \leq L(F) + \langle g, f \rangle + \frac{S}{2}\|f\|^2$ from \cref{lemmaStandardSmoothnessUpperBound}, we substitute the weak update $f = \eta f^w_{k+1}$. Using strong convexity to write $\|u\|^2 \leq \frac{1}{\mu}\|u\|_{H_k}^2$, and the cosine angle properties from \cref{lemmaCosineAngleProperties}, we obtain
    \begin{align}
        L(F_k) -L(F_{k+1})
        &\geq - \eta \langle g_k, f_{k+1}^w \rangle - \frac{S \eta^2}{2} \|f_{k+1}^w\|^2 \\
        &\geq \eta \|f_{k+1}^w\|_{H_k}^2 - \frac{S \eta^2}{2\mu} \|f_{k+1}^w\|_{H_k}^2 \\
        &= \left(\eta -\frac{S \eta^2}{2\mu} \right)\Theta_k^2\|f_{k+1}\|_{H_k}^2.
    \end{align}
    Using $\Theta_k\geq\Theta$ and applying \cref{defHessianDominatedLoss}, we obtain
    \begin{align*}
        L(F_k)-L(F_{k+1})
        &\geq \eta\left(1-\frac{\eta S}{2\mu}\right)\Theta^2\|f_{k+1}\|_{H_k}^2 \\
        &\geq \rho\bigl(L(F_k)-L_\star\bigr).
    \end{align*}
    Rearranging and subtracting $L_\star$ gives
    \[
        L(F_{k+1})-L_\star \leq (1-\rho)\bigl(L(F_k)-L_\star\bigr).
    \]
    Iterating yields
    \[
        L(F_k)-L_\star \leq (1-\rho)^k\bigl(L(F_0)-L_\star\bigr),
    \]
    proving linear convergence of the optimality gap.
\end{proof}

\subsection{Properties of Gradient Regularization}
We begin by establishing some basic properties the gradient regularized scheme. In the subsequent analysis, define $K_k := H_k + \lambda_k I$.

\begin{lemma}\label{lemmaLambdafMinusGkHk}
    The exact regularized update \eqref{eqRegBoostingRoundExact} satisfies
    \begin{equation}\label{eqRegNewtonUpdateDirection}
        \lambda_k f_{k+1} = -\left( g_k + H_k[f_{k+1}] \right).
    \end{equation}
\end{lemma}
\begin{proof}
    Multiply the Newton update \eqref{eqRegBoostingRoundExact} by $(H_k + \lambda_k I)$.
\end{proof}

\begin{lemma}\label{lemmaUsefulProperties}
    If $L$ is convex, then for all $u\in \mathcal{H}$
    \begin{align*}
        \|u\|_{K_k}^2  
        &:= \langle(H_k + \lambda_k I)u, u\rangle \\
        &=  \lambda_k \|u\|^2 + \langle u, H_k u\rangle \\
        &\geq \lambda_k \|u\|^2.
    \end{align*}
    Furthermore, if the weak learner family $\mathcal{F}$ is closed under scalar multiplication, then the iterates of \cref{alg:boostingGradRegNewton} satisfy
    \begin{enumerate}
        \item $ \|f_{k+1}^w\|_{K_k}^2 = -\langle g_k, f_{k+1}^w \rangle$,
        \item $\lambda_k \|f_{k+1}^w\| \leq \|g_k\|$.
    \end{enumerate}
\end{lemma}
\begin{proof}
    The first result follows by convexity. For the rest, let
    $$\phi(t) := Q_k^{reg}(t f_{k+1}^w) = \frac{1}{2}t^2 \|f_{k+1}^w\|_{K_k}^2 + t \langle g_k, f_{k+1}^w\rangle.$$
    Since $\mathcal{F}$ is scalable, $\phi$ is minimized at $t=1$. Therefore, $0=\phi'(1)=\|f_{k+1}^w\|_{K_k}^2+\langle g_k,f_{k+1}^w\rangle$. As for the last inequality, this follows by the first inequality and Cauchy--Schwarz:
    \begin{align*}
        \lambda_k \|f_{k+1}^w\|^2 \leq \|f_{k+1}^w\|_{K_k}^2 = -\langle g_k, f_{k+1}^w \rangle \leq \|g_k\| \|f_{k+1}^w\|. 
    \end{align*}
\end{proof}

\begin{proof}[Proof of \cref{lemmaDecreaseInLossNEW}]
Write $s=f_{k+1}^w$. By \cref{lemmaStandardHessianResults,lemmaUsefulProperties},
\begin{align*}
    L(F_{k+1})-L(F_k)
    &\leq \eta\langle g_k,s\rangle
    +\frac{\eta^2}{2}\langle H_k s,s\rangle
    +\frac{M\eta^3}{3}\|s\|^3 \\
    &=-\left(\eta-\frac{\eta^2}{2}\right)\|s\|_{K_k}^2
    -\frac{\eta^2}{2}\lambda_k\|s\|^2
    +\frac{M\eta^3}{3}\|s\|^3.
\end{align*}
Since $\lambda_k\|s\|\leq\|g_k\|$ and $\lambda_k^2\geq M\|g_k\|$, we have
\[
    M\|s\|^3
    \leq\frac{M\|g_k\|}{\lambda_k}\|s\|^2
    \leq\lambda_k\|s\|^2.
\]
Consequently,
\begin{align*}
    L(F_k)-L(F_{k+1})
    &\geq\left(\eta-\frac{\eta^2}{2}\right)\|s\|_{K_k}^2
    +\left(\frac{\eta^2}{2}-\frac{\eta^3}{3}\right)\lambda_k\|s\|^2 \\
    &\geq\left(\eta-\frac{\eta^3}{3}\right)\lambda_k\|s\|^2,
\end{align*}
where the last inequality uses $\|s\|_{K_k}^2\geq\lambda_k\|s\|^2$ and $\eta-\eta^2/2\geq0$ for $\eta\in(0,2]$.
\end{proof}

For ease of exposition, we restate some results of the relevant lemmas and theorems from the main body of the paper.

\begin{lemma}\label{appendixLemmaGradientGrowth}
    If the iterates of \cref{alg:boostingGradRegNewton} have weak gradient edge $\gamma$, with learning rate $\eta\in(0,1]$, then the gradient growth is bounded by
    \begin{align}
        \|g_{k+1}\| 
        &\leq \left(1-\eta(1-\sqrt{1-\gamma^2})\right)\|g_k\| + (\eta+\eta^2)\lambda_k \|f_{k+1}^w\| \\
        &\leq \left(1 + \eta^2 + \eta \sqrt{1-\gamma^2}\right)\|g_k\|.
    \end{align}
\end{lemma}
\begin{proof}
    Using \cref{eqWeakGradient}, and \cref{lemmaLambdafMinusGkHk,lemmaStandardHessianResults}, we can write
    \begin{align}
        g_{k+1} 
        &= g_k + \eta H_k f^w_{k+1} + E \\
        &= g_k + \eta (-g_k^w-\lambda_k f_{k+1}^w) + E \\
        &= (1-\eta) g_k + \eta(g_k - g_k^w) - \eta \lambda_k f^w_{k+1} + E
    \end{align}
    where $\|E\| \leq M\eta^2 \|f^w_{k+1}\|^2$. Hence using the bounds in \cref{lemmaUsefulProperties} we obtain that
    \begin{align}
        \|g_{k+1}\| 
        &= \| (1-\eta) g_k + \eta(g_k - g_k^w) - \eta\lambda_k f^w_{k+1} + E \| \\
        &\leq (1-\eta) \|g_k\| + \eta \|g_k - g_k^w\| +  \eta \lambda_k \|f^w_{k+1}\| + M \eta^2 \|f^w_{k+1}\|^2 \label{eqGradientRecursionForLater}\\
        &\leq (1-\eta) \|g_k\| + \eta \sqrt{1-\gamma^2}\|g_k\| +  \eta \lambda_k \|f^w_{k+1}\| + \eta^2 \lambda_k \|f^w_{k+1}\| \\
        &= \left(1-\eta(1-\sqrt{1-\gamma^2})\right)\|g_k\| + (\eta+\eta^2)\lambda_k \|f_{k+1}^w\| \\
        &\leq \left(1 + \eta^2 + \eta \sqrt{1-\gamma^2}\right)\|g_k\|.
    \end{align}
\end{proof}

\subsection{Proof of Global Convergence for the Gradient Regularized Scheme with Weak Learners}

The following global convergence proof of the gradient regularized Newton scheme generalizes the proof of \citet{2023RegularizedNewtonMethodWithGlobalConvergence} to the setting of Restricted Newton Descent with a constant learning rate $\eta \in (0, 1]$ and weak learners. The primary challenge lies in handling the additional error terms introduced by the learning rate and the inexact weak iterates $f^w_{k+1}$, which we handled in our previous Lemmas in this paper. Additionally, we adapt the original proof's fixed subsequence condition and fraction to depend on both the learning rate $\eta$ and weak gradient edge $\gamma$.

\begin{theorem}\label{theoremAppendixGradRegGlobalConv}
    Assume $L$ is convex with $2M$-Lipschitz Hessian, a bounded initial sublevel set with $\|F-F^*\|\leq\Lambda$, and weak gradient edge $\gamma$. Let the regularization be chosen as $\lambda_k = C \sqrt{M\|g_k\|}$ for some $C\geq 1$. For any learning rate $\eta \in (0,1]$, the gradient regularized Newton scheme has the global convergence rate
   \begin{align}
        L(F_k) - L(F^*) \leq \max\bigg\{ 
            A(k, \Lambda, \eta, \gamma, g_0), \quad 
            B(k, \Lambda, \eta, \gamma, M, C)
        \bigg\}.
    \end{align}
    where
    \begin{equation}
        A(k, \Lambda, \eta, \gamma, g_0) =  \underbrace{\Lambda \|g_0\| \exp\left( - \frac{k\eta}{8}(1-\sqrt{1-\gamma^2})(3-\sqrt{1-\gamma^2}-2\eta) \right)}_{\text{Linear Regime}},
    \end{equation}
    and
    \begin{equation}
        B(k, \Lambda, \eta, \gamma, M, C) =\underbrace{\frac{1024 C^2 M \Lambda^3}{\eta^2 (1-\sqrt{1-\gamma^2})^6(2+k)^2}\frac{(1+\eta)^4}{(1-\frac{1}{3}\eta^2)^2}}_{\mathcal{O}\left(\frac{1}{k^2}\right)\text{ Regime}}.
    \end{equation}
\end{theorem}
\begin{proof}
    By \cref{lemmaDecreaseInLossNEW} we have that $L(F_k) \leq L(F_{k-1}) \leq ... \leq L(F_0)$, implying that the iterates $F_k$ remain within the sublevel set defined by $F_0$. Consequently, $\|F_k-F^*\| \leq \Lambda$ for all $k$. By convexity we have the bound
    \begin{align}
        L(F_k) - L(F^*)
        &\leq \langle g_k, F_k-F^*\rangle  \leq \|g_k\| \Lambda.
    \end{align}
    Let $\mathcal{I}_\infty := \{ t\in \mathbb{Z}_{\geq0} : \|g_{t+1}\|\geq (1-\eta a) \|g_t\| \}$ be the set of indices where the gradient does not contract sufficiently, for some $0<a<1$ to be specified later. Consider first any $k\in \mathcal{I}_\infty$. It follows from \cref{appendixLemmaGradientGrowth} that
    \begin{align}
        (1-\eta a) \|g_k\| \leq \|g_{k+1}\| \leq  \left(1-\eta(1-\sqrt{1-\gamma^2})\right)\|g_k\| + (\eta+\eta^2)\lambda_k \|f_{k+1}^w\|
    \end{align}
    Rearranging the terms we obtain that
    \begin{align}\label{eqCase1LambdafNormGradientBound}
        \lambda_k \|f_{k+1}^w\| \geq \frac{1 - \sqrt{1-\gamma^2} - a}{1+\eta} \|g_k\|
    \end{align}
    For the next argument we require the bound \eqref{eqCase1LambdafNormGradientBound} to be positive, hence we choose 
    \begin{align}
        a := \frac{1-\sqrt{1-\gamma^2}}{2}.
    \end{align}
    By \eqref{eqCase1LambdafNormGradientBound} and \cref{lemmaDecreaseInLossNEW} we have
    \begin{align}
        L(F_k) - L(F_{k+1})
        &\geq \left( 1 - \frac{1}{3}\eta^2 \right) \eta  \lambda_k \|f_{k+1}^w\|^2 \\
        &\geq \left( 1 - \frac{1}{3}\eta^2 \right) a^2 \frac{\eta}{(1+\eta)^2} \frac{\|g_k\|^2}{\lambda_k} \\
        &= \left( 1 - \frac{1}{3}\eta^2 \right) a^2  \frac{\eta}{(1+\eta)^2} \frac{\|g_k\|^\frac{3}{2}}{C\sqrt{M}} \\
        &\geq \left( 1 - \frac{1}{3}\eta^2 \right) \frac{a^2 }{C\sqrt{M\Lambda^3}} \frac{\eta}{(1+\eta)^2} \bigg(  L(F_k) - L(F^*) \bigg)^\frac{3}{2} \\
        &= \tau \bigg(  L(F_k) - L(F^*) \bigg)^\frac{3}{2}
    \end{align}
    where $\tau = \left( 1 - \frac{1}{3}\eta^2 \right) \frac{a^2 }{C\sqrt{M\Lambda^3}} \frac{\eta}{(1+\eta)^2}$. If this recursion held at every iteration, the $\mathcal{O}(\frac{1}{k^2})$ rate would follow by \cref{lemmaAlpha3div2Recursion}. Instead, we use the subsequence technique of \cite{2023RegularizedNewtonMethodWithGlobalConvergence}.

    Fix any $k\geq1$ and define
    \[
        \mathcal{I}_k
        := \mathcal{I}_\infty\cap\{0,\ldots,k-1\},
        \qquad m:=|\mathcal{I}_k|.
    \]
    We consider the cases $m\geq pk$ and $m<pk$, for some fraction $0<p<1$ to be specified later.

    \textbf{Case} $|\mathcal{I}_k|\geq pk$:
    Here $m\geq1$. Enumerate the selected indices as $t_0<\cdots<t_{m-1}$ and define
    \[
        \alpha_j:=L(F_{t_j})-L(F^*)
        \quad (0\leq j<m),
        \qquad
        \alpha_m:=L(F_k)-L(F^*).
    \]
    By monotonicity of the loss and the decrease bound at each selected index, for $0\leq j<m$ we have
    \begin{align}
        \alpha_{j+1}
        &\leq L(F_{t_j+1})-L(F^*) \\
        &= \alpha_j+L(F_{t_j+1})-L(F_{t_j}) \\
        &\leq \alpha_j-\tau\alpha_j^{3/2}.
    \end{align}
    If $\alpha_m=0$, the desired bound is immediate. Otherwise, applying \cref{lemmaAlpha3div2Recursion} to these $m$ steps gives
    \begin{align}
        L(F_k)-L(F^*)
        &= \alpha_m \\
        &\leq \frac{4}{\tau^2(m+2)^2} \\
        &\leq \frac{4}{\tau^2p^2(k+2)^2} \\
        &= \frac{4C^2M\Lambda^3}
                 {\eta^2a^4p^2(k+2)^2}
           \frac{(1+\eta)^4}
                {(1-\frac{1}{3}\eta^2)^2}.
    \end{align}
    The second inequality uses $m\geq pk$ and $p<1$, which imply $m+2\geq p(k+2)$. Consequently, this case gives a global $\mathcal{O}(\frac{1}{k^2})$ rate.
    
    \textbf{Case} $|\mathcal{I}_k| < pk$:
    For indices $t\in\{0,\ldots,k-1\}\setminus\mathcal{I}_k$, we have the contraction $\|g_{t+1}\| \leq (1-\eta a) \|g_t\|$. For $t \in \mathcal{I}_k$ the gradient may grow, but it is bounded by \cref{appendixLemmaGradientGrowth} via $\|g_{t+1}\| \leq \left(1 + \eta^2 + \eta \sqrt{1-\gamma^2}\right) \|g_t\|$. Combining both we obtain that
    \begin{align}
        \|g_k\| 
        &\leq \left(1 + \eta^2 + \eta \sqrt{1-\gamma^2}\right)^{|\mathcal{I}_k|} (1-\eta a)^{k-|\mathcal{I}_k|} \|g_0\|\\
        &\leq \left(1 + \eta^2 + \eta \sqrt{1-\gamma^2}\right)^{pk}(1-\eta a)^{(1-p)k} \|g_0\|.
    \end{align}
    Using the bound $1+x \leq e^x$ and taking logarithms, we see that this scheme produces a contraction if
    \begin{align}
        p\left( \eta^2 + \eta \sqrt{1-\gamma^2} \right) - (1-p)\eta a < 0.
    \end{align}
    Rearranging the above, we find that we require
    \begin{align}
        p < \frac{a}{\eta + a + \sqrt{1-\gamma^2}}.
    \end{align}
    Substituting our choice $a = (1-\sqrt{1-\gamma^2})/2$, the denominator simplifies to $\eta + (1+\sqrt{1-\gamma^2})/2$. Since $\sqrt{1-\gamma^2} \leq 1$, this denominator is strictly less than $\eta + 1$. Therefore, for any $\eta \in (0, 1]$, the choice $p = \frac{a}{2}$ is sufficient. In general, one can show that $a$ and $p$ need to satisfy the inequality
    \begin{align}
        \frac{p\left(\eta+\sqrt{1-\gamma^2}\right)}{1-p} < a <1-\sqrt{1-\gamma^2}.
    \end{align}
    for the proof in both cases above to be valid. With the choice $4p = 2a = 1-\sqrt{1-\gamma^2}$, one finds through elementary algebra and the bound $1+x \leq e^x$ that the gradient growth can be bounded by
    \begin{align}
        \|g_k\| \leq \exp\left( - \frac{k\eta}{8}(1-\sqrt{1-\gamma^2})(3-\sqrt{1-\gamma^2}-2\eta) \right) \|g_0\|,
    \end{align}
    yielding a linear rate of convergence.
    
    \textbf{Final bound:} The final convergence rate is obtained by taking the maximum bound of the two cases above. Note that $1-\sqrt{1-\gamma^2} = \frac{\gamma^2}{1+\sqrt{1-\gamma^2}} \geq \frac{\gamma^2}{2}$. Hence the global convergence rate in the first case is 
    \begin{equation}
        \mathcal{O}\left(\frac{C^2M \Lambda^3}{\eta^2 k^2 \gamma^{12}}\right),
    \end{equation}
    while for the second case this is 
    \begin{equation}
        \mathcal{O}\left(\exp\left(-k\eta\gamma^2\left(\frac{\gamma^2}{2} +2-2\eta\right) / 16\right)\right).
    \end{equation}
    This concludes the proof.
\end{proof}

The following lemma used in the proof of \cref{theoremAppendixGradRegGlobalConv} is a standard result in the convex optimization literature, see e.g. \cite{2006CubicRegularizationOfNewtonMethodAndItsGlobalPerformance}. We include the general proof here for completeness

\begin{lemma}\label{lemmaAlpha3div2Recursion}
    If $\alpha_k$ is a decreasing sequence satisfying $0 \leq \alpha_{k+1} \leq \alpha_k - c\alpha_k^{\frac{3}{2}}$ for some $c>0$, then for all $k \geq 0$:
    \begin{equation}
        \alpha_k \leq \frac{1}{\left( \frac{1}{\sqrt{\alpha_0}} + \frac{c}{2}k \right)^2} \leq \frac{4}{c^2 (k + 2)^2}.
    \end{equation}
\end{lemma}
\begin{proof}
    Using the hypothesis $\alpha_k - \alpha_{k+1} \geq c \alpha_k^{3/2}$, consider the increment
    \begin{align}
        \frac{1}{\sqrt{\alpha_{k+1}}} - \frac{1}{\sqrt{\alpha_k}}
        &= \frac{\sqrt{\alpha_k} - \sqrt{\alpha_{k+1}}}{\sqrt{\alpha_k \alpha_{k+1}}} \\
        &= \frac{\alpha_k - \alpha_{k+1}}{\sqrt{\alpha_k \alpha_{k+1}} (\sqrt{\alpha_k} + \sqrt{\alpha_{k+1}})} \\
        &\geq \frac{c \alpha_k^{3/2}}{\sqrt{\alpha_k \alpha_{k+1}} (\sqrt{\alpha_k} + \sqrt{\alpha_{k+1}})}.
    \end{align}
    Since the sequence is decreasing ($\alpha_{k+1} \leq \alpha_k$), we have $\sqrt{\alpha_{k+1}} \leq \sqrt{\alpha_k}$, hence
    \begin{align}
        \frac{1}{\sqrt{\alpha_{k+1}}} - \frac{1}{\sqrt{\alpha_k}}
        \geq \frac{c \alpha_k^{3/2}}{\sqrt{\alpha_k^2} (\sqrt{\alpha_k} + \sqrt{\alpha_k})}
        = \frac{c}{2}.
    \end{align}
    Summing from $i=0$ to $k-1$ yields a telescoping sum, from which we obtain that
    \begin{equation}
        \frac{1}{\sqrt{\alpha_k}} \geq \frac{1}{\sqrt{\alpha_0}} + \frac{ck}{2}.
    \end{equation}
    Inverting and squaring both sides yields the inequality
    \begin{equation}
        \alpha_k \leq \frac{1}{\left( \frac{1}{\sqrt{\alpha_0}} + \frac{c}{2}k \right)^2}.
    \end{equation}
    The final bound follows from the fact that $c\sqrt{\alpha_k} \leq 1$, which follows directly from the hypothesis $0 \leq \alpha_{k+1} \leq \alpha_k - c\alpha_k^{\frac{3}{2}}$.
\end{proof}

\subsection{Proof of Local Convergence for GRN with weak learners}

\begin{proof}[Proof of \cref{thm:local_rate}]
    We start by using the triangle inequality and the definition of the weak gradient edge $\gamma$ and write
    \begin{align}
        \|g^w_k\| 
        &\leq \|g_k\| + \|g_k^w - g_k\| \\
        &\leq \left( 1 + \sqrt{1-\gamma^2} \right) \|g_k\| \\
        &:= C_\gamma \|g_k\|.
    \end{align}
    Since we assumed $\mu$-strong convexity, we have that
    \begin{align}
        \|f^w_{k+1}\| 
        &= \|(H_k+\lambda_k I)^{-1} g_k^w\| \\
        &\leq \frac{1}{\mu}\|g_k^w\| \\
        &\leq \frac{C_\gamma}{\mu}\|g_k\|.
    \end{align}
    We combine the above with \cref{lemmaUsefulProperties,appendixLemmaGradientGrowth}, specifically \cref{eqGradientRecursionForLater}, to obtain that
    \begin{align}
        \|g_{k+1}\|
        &\leq (1-\eta) \|g_k\| + \eta \|g_k - g_k^w\| +  \eta \lambda_k \|f^w_{k+1}\| + M \eta^2 \|f^w_{k+1}\|^2 \\
        &\leq (1-2\eta + \eta C_\gamma) \|g_k\| + \eta \lambda_k \|f^w_{k+1}\| + M \eta^2 \|f^w_{k+1}\|^2  \\
        &\leq (1-2\eta + \eta C_\gamma) \|g_k\| + \sqrt{C^2 M \eta^2 \frac{C_\gamma^2}{\mu^2}} \|g_k\|^\frac{3}{2} + C^2 M \eta^2 \frac{C_\gamma^2}{\mu^2}\|g_k\|^2.
    \end{align}
    Write $\rho = (1-2\eta + \eta C_\gamma)$ and $K = C^2 M \eta^2 \frac{C_\gamma^2}{\mu^2}$. We seek the condition for a strict decrease $\| g_{k+1} \| < \|g_k\|$. Substituting the above and dividing by $\|g_k\|$ we obtain that
    \begin{align}
        \sqrt{K \|g_k\|} + K\|g_k\| < 1-\rho.
    \end{align}
    Let $\Delta = 1 - \rho = \eta(1 - \sqrt{1-\gamma^2})$. Since $\eta \leq 1$, we observe that $\Delta \leq 1$. Let $y = \sqrt{K\|g_k\|}$. The condition becomes $y + y^2 < \Delta$. A sufficient condition for this to hold is $y \leq \frac{\Delta}{2}$, because this implies $y+y^2 \leq \frac{\Delta}{2} + \frac{\Delta^2}{4} = \frac{\Delta(2+\Delta)}{4}$. Since $\Delta \leq 1$, we have $(2+\Delta) \leq 3$, and thus $y+y^2 \leq \frac{3}{4}\Delta < \Delta$. Expanding the condition $y \leq \frac{\Delta}{2}$ yields
    \begin{align}
        \sqrt{\frac{C^2 M \eta^2 C_\gamma^2}{\mu^2} \|g_k\|} \leq \frac{\eta(1 - \sqrt{1-\gamma^2})}{2}.
    \end{align}
    Canceling $\eta$ and rearranging for $\sqrt{\|g_k\|}$, we obtain
    \begin{align}
        \sqrt{\|g_k\|} \leq \frac{\mu}{2C\sqrt{M}} \frac{1 - \sqrt{1-\gamma^2}}{C_\gamma} = \frac{\mu}{2C\sqrt{M}} \frac{1 - \sqrt{1-\gamma^2}}{1 + \sqrt{1-\gamma^2}}.
    \end{align}
    Squaring both sides gives the bound in the hypothesis. Under this condition, the sum of the higher order terms is bounded by $\frac{3}{4}\Delta \|g_k\| = \frac{3}{4}(1-\rho)\|g_k\|$. Plugging this back into the gradient bound we obtain that
    \begin{align}
        \|g_{k+1}\| \leq \rho \|g_k\| + \frac{3}{4}(1-\rho)\|g_k\| = \frac{3+\rho}{4}\|g_k\|.
    \end{align}
    The conclusion follows since
    \begin{equation}
        1-\sqrt{1-\gamma^2} 
        = 
        \frac{\gamma^2}{1+\sqrt{1-\gamma^2}} 
        \geq
        \frac{\gamma^2}{2}
    \end{equation}
    whence
    \begin{equation}
        \frac{3+\rho}{4} 
        = 
        1 - \frac{\eta}{4} \left( 1-\sqrt{1-\gamma^2} \right)
        \leq
        1-\frac{\eta\gamma^2}{8}
    \end{equation}
    Since this bound preserves the gradient threshold in the hypothesis, induction from $k_0$ gives the claimed contraction for every $k\geq k_0$.
\end{proof}

\begin{remark}
    If $\eta=1$ and $\gamma=1$, the coefficient of the linear $\|g_k\|$ term vanishes, allowing for a superlinear rate with exponent $\frac{3}{2}$, recovering the classical result of \citet{2023RegularizedNewtonMethodWithGlobalConvergence}. However, when $\eta < 1$ or $\gamma < 1$, this term is non-zero, and we are only able to prove a linear rate. We conjecture that superlinear rates are generally unattainable in Newton schemes involving weak iterates, suggesting an interesting direction for future research.
\end{remark}

\clearpage
\section{Divergence of Vanilla Newton on Smooth Convex Losses}\label{app:newton_divergence}
This appendix records explicit smooth convex objectives for which the classical
(unregularized) Newton step may fail to converge globally, even though the loss
satisfies the Lipschitz-Hessian assumption used throughout \cref{secGradientRegularizedNewton} and
\cref{app:GlobalConvergenceProofs}. These examples show that a fixed damping (constant learning rate)
does not, by itself, globalize Newton's method, and motivate the adaptive
gradient regularization used in \cref{alg:boostingGradRegNewton}.

\subsection{Newton, damped Newton, and gradient regularized Newton}

\begin{definition}[Newton-type updates in $\mathbb{R}^d$]
\label{def:newton_updates}
Let $L:\mathbb{R}^d\to\mathbb{R}$ be twice continuously differentiable and assume
$\nabla^2 L(x)\succ 0$ at the current iterate.
\begin{itemize}
  \item The (classical) Newton update is
  \[
    N(x) := x - (\nabla^2 L(x))^{-1}\nabla L(x).
  \]
  \item Given a fixed learning rate $\eta>0$, the damped Newton update is
  \[
    N_\eta(x) := x - \eta(\nabla^2 L(x))^{-1}\nabla L(x).
  \]
\end{itemize}
Assume moreover that $\nabla^2 L$ is $2M$-Lipschitz in operator norm:
\[
  \|\nabla^2 L(u)-\nabla^2 L(v)\|_{\mathrm{op}}\le 2M\|u-v\|
  \qquad \forall u,v\in\mathbb{R}^d.
\]
Then the (exact) gradient regularized Newton (GRN) update is
\[
  \mathrm{GRN}_{M,\eta}(x) := x - \eta\big(\nabla^2L(x)+\lambda(x)I\big)^{-1}\nabla L(x),
  \qquad \lambda(x) := \sqrt{M\|\nabla L(x)\|}.
\]
\end{definition}

\subsection{A canonical 1D counterexample: $L(x)=\sqrt{1+x^2}$}

\begin{proposition}[A smooth convex loss with divergent Newton (1D)]
\label{prop:sqrt_divergence}
Let $L:\mathbb{R}\to\mathbb{R}$ be given by $L(x):=\sqrt{1+x^2}$.
Then $L$ is convex and smooth, and $L''$ is globally Lipschitz.
The classical Newton iteration $x_{k+1}=N(x_k)$ satisfies the exact recursion
\[
  x_{k+1} = -x_k^3.
\]
In particular, $|x_{k+1}|=|x_k|^3$, yielding cubic local convergence for $|x_0|<1$
and divergence for $|x_0|>1$.
\end{proposition}

\begin{proof}
We compute
\[
L'(x)=\frac{x}{\sqrt{1+x^2}},\qquad
L''(x)=\frac{1}{(1+x^2)^{3/2}},\qquad
L^{(3)}(x)=-\frac{3x}{(1+x^2)^{5/2}}.
\]
In particular $L''(x)>0$ for all $x$, hence $L$ is convex. Since $L^{(3)}$ is bounded,
$L''$ is globally Lipschitz. Moreover,
\[
\frac{L'(x)}{L''(x)}
=\frac{x}{\sqrt{1+x^2}}\,(1+x^2)^{3/2}
=x(1+x^2).
\]
Thus the Newton update is
\[
x_{k+1}=x_k-\frac{L'(x_k)}{L''(x_k)}
=x_k-x_k(1+x_k^2)
=-x_k^3.
\]
Finally, $|x_{k+1}|=|x_k|^3$ implies $x_k\to 0$ if $|x_0|<1$ and $|x_k|\to\infty$
if $|x_0|>1$.
\end{proof}

\begin{proposition}[Fixed damping does not globalize Newton]
\label{prop:damped_not_global}
Fix any learning rate $\eta>0$ and consider $x_{k+1}=N_\eta(x_k)$ for
$L(x)=\sqrt{1+x^2}$. Define
\[
  X(\eta) := \sqrt{\max\Big\{0,\frac{3}{\eta}-1\Big\}}.
\]
Then every nonzero initialization with $|x_0|\ge X(\eta)$ diverges and satisfies
\[
  |x_{k+1}|\ge 2|x_k| \qquad \forall k\ge 0.
\]
\end{proposition}

\begin{proof}
From the computations above, $L'(x)/L''(x)=x(1+x^2)$, hence
\[
x_{k+1}=x_k-\eta x_k(1+x_k^2)=x_k\big(1-\eta(1+x_k^2)\big).
\]
If $|x_k|\ge X(\eta)$, then by definition $\eta(1+x_k^2)\ge 3$, so
$|1-\eta(1+x_k^2)|\ge 2$ and therefore $|x_{k+1}|\ge 2|x_k|$.
Induction yields divergence.
\end{proof}

\subsection{A drift-based generator for smooth convex Newton counterexamples}
\label{app:drift_generator}

The previous example is not accidental. We now record a general construction
principle: prescribing a one-dimensional \emph{Newton drift} $d(x)$ determines a
convex objective $L$ whose Newton iteration is exactly $x_{k+1}=x_k-\tilde d(x_k)$,
where $\tilde d$ is the odd extension of $d$. This yields an infinite family of
smooth convex losses on which classical Newton diverges.

\begin{theorem}[Lipschitz-Hessian drift-to-loss generator for 1D Newton dynamics]
\label{thm:drift_to_loss_lipschitz}
Let $d:[0,\infty)\to[0,\infty)$ be continuous and satisfy:
{\renewcommand{\labelenumi}{(D\arabic{enumi})}%
\begin{enumerate}
\setcounter{enumi}{-1}
\item $d(0)=0$ and $d(x)>0$ for all $x>0$;
\item there exist $\delta>0$ and $C_0<\infty$ such that
\[
|d(x)-x|\le C_0x^2 \qquad \forall x\in[0,\delta].
\]
\end{enumerate}}

Define, for $x>0$,
\[
\Phi(x):=\int_1^x \frac{ds}{d(s)},\qquad
L_+'(x):=a\,e^{\Phi(x)}\ (a>0),\qquad
L_+(x):=\int_0^x L_+'(t)\,dt,
\]
and extend evenly by $L(x):=L_+(|x|)$. Let $\tilde d:\mathbb{R}\to\mathbb{R}$ be the odd extension
$\tilde d(x):=\mathrm{sign}(x)\,d(|x|)$ with $\tilde d(0)=0$. Then:
{\renewcommand{\labelenumi}{(\roman{enumi})}%
\begin{enumerate}
\item $L$ is even, convex, and $C^2(\mathbb{R})$, with $L''(x)>0$ for $x\neq 0$.
Moreover, for $x>0$,
\[
L'(x)=L_+'(x),\qquad
L''(x)=\frac{L_+'(x)}{d(x)}.
\]
In particular, $L'$ is odd and $L''$ is even.

\item For every $x\neq 0$,
\[
\frac{L'(x)}{L''(x)}=\tilde d(x),
\]
and therefore the classical Newton update for $L$ is exactly
\[
N(x)=x-\frac{L'(x)}{L''(x)}=x-\tilde d(x).
\]

\item If there exist $X_0>0$ and $\beta>1$ such that
\[
d(x)\ge (1+\beta)x \qquad \forall x\ge X_0,
\]
then every Newton iterate initialized with $|x_0|\ge X_0$ diverges and satisfies
\[
|x_{k+1}|\ge \beta|x_k| \qquad \forall k\ge 0.
\]

\item \textbf{(Globally Lipschitz Hessian.)}
Assume in addition that $d\in C^1((0,\infty))$ and that
\begin{equation}
\label{eq:D2_ratio}
C_1 \;:=\;\sup_{x>0}\frac{|d'(x)-1|}{d(x)}\;<\;\infty.
\end{equation}
Assume moreover that there exist $\kappa>1$ and $X_1>0$ such that
\begin{equation}
\label{eq:D3_linear_lower}
d(x)\ge \kappa x \qquad \forall x\ge X_1.
\end{equation}
(For instance, \cref{eq:D3_linear_lower} holds whenever the hypothesis in (iii) holds.)
Then $L''$ is globally Lipschitz on $\mathbb{R}$, i.e.\ there exists $M<\infty$ such that
\[
|L''(u)-L''(v)|\le 2M|u-v|\qquad \forall u,v\in\mathbb{R}.
\]
More precisely, for $x>0$ we have
\begin{equation}
\label{eq:L3_formula}
L^{(3)}(x)=\frac{L_+'(x)}{d(x)^2}\,\bigl(1-d'(x)\bigr)
= L''(x)\,\frac{1-d'(x)}{d(x)},
\end{equation}
and hence $\sup_{x>0}|L^{(3)}(x)|<\infty$, so one may take
$2M=\sup_{x\neq 0}|L^{(3)}(x)|$.
\end{enumerate}}
\end{theorem}

\begin{proof}
We split the argument into four parts.

\paragraph{Part A: construction and smoothness on $(0,\infty)$.}
By (D0), the map $s\mapsto 1/d(s)$ is continuous on compact subsets of $(0,\infty)$.
Hence $\Phi\in C^1((0,\infty))$ with $\Phi'(x)=1/d(x)$.
Therefore $L_+'(x)=a e^{\Phi(x)}$ is $C^1((0,\infty))$ and
\[
(L_+')'(x)=a e^{\Phi(x)}\Phi'(x)=\frac{L_+'(x)}{d(x)}.
\]
Thus $L_+\in C^2((0,\infty))$, with
\[
L_+''(x)=(L_+')'(x)=\frac{L_+'(x)}{d(x)}>0 \qquad (x>0),
\]
so $L_+$ is strictly convex on $(0,\infty)$.

\paragraph{Part B: behavior at $0$ and $C^2$ even extension.}
Write $d(x)=x(1+\varepsilon(x))$ on $(0,\delta]$, where by (D1) we have
$|\varepsilon(x)|\le C_0x$, hence $\varepsilon(x)\to 0$ as $x\downarrow 0$.
Then on $(0,\delta]$,
\[
\frac{1}{d(x)}=\frac{1}{x}\cdot \frac{1}{1+\varepsilon(x)}
=\frac{1}{x}+g(x),
\]
where $g$ is bounded on $(0,\delta]$ (since $(1+\varepsilon)^{-1}=1+O(\varepsilon)$).
Integrating gives
\[
\Phi(x)=\int_1^x \frac{ds}{d(s)}
=\int_1^x \frac{ds}{s}+\int_1^x g(s)\,ds
=\log x + C_\ast + o(1)\qquad (x\downarrow 0)
\]
for some finite constant $C_\ast$.
Exponentiating yields
\[
L_+'(x)=a e^{\Phi(x)}=a e^{C_\ast}\,x(1+o(1))\qquad (x\downarrow 0),
\]
so $L_+'(x)\to 0$ and $L_+'(x)/x\to a e^{C_\ast}\in(0,\infty)$.
Since $d(x)=x(1+o(1))$, we obtain
\[
L_+''(x)=\frac{L_+'(x)}{d(x)}\to a e^{C_\ast}\in(0,\infty)\qquad (x\downarrow 0).
\]
Define $L_+(0):=0$, $L_+'(0):=0$, and $L_+''(0):=\lim_{x\downarrow 0}L_+''(x)$; then
$L_+\in C^2([0,\infty))$.

Now define $L(x)=L_+(|x|)$. This $L$ is even and $C^2(\mathbb{R})$ with
$L'(x)$ odd and $L''(x)$ even, and for $x>0$ we have
$L'(x)=L_+'(x)$ and $L''(x)=L_+''(x)=L_+'(x)/d(x)$.
Since $L''(x)\ge 0$ everywhere and $L''(x)>0$ for $x\neq 0$, $L$ is convex.
This proves (i).

\paragraph{Part C: drift identity and divergence.}
For $x>0$, by (i) we have $L'(x)=L_+'(x)$ and $L''(x)=L_+'(x)/d(x)$, hence
\[
\frac{L'(x)}{L''(x)}=d(x).
\]
By evenness of $L''$ and oddness of $L'$, the ratio $L'/L''$ is odd, so for $x<0$,
\[
\frac{L'(x)}{L''(x)}=-\frac{L'(|x|)}{L''(|x|)}=-d(|x|)=\tilde d(x),
\]
proving (ii). For the divergence statement (iii),
if $x\ge X_0$, then $N(x)=x-d(x)\le -\beta x$, hence $|N(x)|\ge \beta|x|$, and similarly for $x\le -X_0$.

\paragraph{Part D: globally Lipschitz Hessian.}
Assume \cref{eq:D2_ratio,eq:D3_linear_lower}.
For $x>0$, differentiating $L''(x)=L_+'(x)/d(x)$ and using $(L_+')'(x)=L_+'(x)/d(x)$
from Part~A gives \cref{eq:L3_formula}. Hence, for $x>0$,
\[
|L^{(3)}(x)| \le \frac{L_+'(x)}{d(x)^2}\,|1-d'(x)|
\le C_1 \frac{L_+'(x)}{d(x)} = C_1 L''(x).
\]
It remains to show $\sup_{x>0}L''(x)<\infty$.
By Part~B, $L''$ extends continuously to $x=0$ with $0<L''(0)<\infty$.
On $[0,X_1]$, continuity gives $\sup_{[0,X_1]}L''<\infty$.
For $x\ge X_1$, using \cref{eq:D3_linear_lower},
\[
\Phi(x)=\Phi(X_1)+\int_{X_1}^x \frac{ds}{d(s)}
\le \Phi(X_1)+\int_{X_1}^x \frac{ds}{\kappa s}
=\Phi(X_1)+\frac{1}{\kappa}\log\frac{x}{X_1},
\]
so $L_+'(x)=a e^{\Phi(x)}\le C\,x^{1/\kappa}$ for some $C$.
Therefore,
\[
L''(x)=\frac{L_+'(x)}{d(x)}\le \frac{C x^{1/\kappa}}{\kappa x}
= \frac{C}{\kappa}x^{\frac{1}{\kappa}-1},
\]
which is bounded for $x\ge X_1$ since $\kappa>1$.
Thus $\sup_{x>0}|L^{(3)}(x)|<\infty$, and since $L''$ is even, $L''$ is globally Lipschitz on $\mathbb{R}$.
\end{proof}

\subsection{Explicit Counterexamples }

\begin{corollary}[Elementary closed-form losses from explicit drifts]
\label{cor:elementary_losses_from_drifts}
Applying \cref{thm:drift_to_loss_lipschitz} with the following choices of $d$
(and the corresponding choice of $a$) yields the stated closed-form losses:
\begin{enumerate}
\item If $d(x)=x(1+x)$ and $a=C/2$, then
\[
L(x)=C\bigl(|x|-\ln(1+|x|)\bigr).
\]

\item If $d(x)=x(1+x^2)$ and $a=C/\sqrt{2}$, then
\[
L(x)=C\bigl(\sqrt{1+x^2}-1\bigr).
\]

\item Fix an integer $m\ge 3$. If
\[
d(x)=\frac{1+x}{m-1}\Bigl((1+x)^{m-1}-1\Bigr)
\qquad\text{and}\qquad
a=C\Bigl(1-2^{-(m-1)}\Bigr),
\]
then
\[
L(x)=C|x|-\frac{C}{m-2}\Bigl(1-(1+|x|)^{-(m-2)}\Bigr).
\]

\item If $d(x)=(1+x^2)\arctan x$ and $a=C(\pi/4)$, then
\[
L(x)=C\Bigl(|x|\arctan|x|-\tfrac12\ln(1+x^2)\Bigr).
\]
\end{enumerate}
In each case, $L$ is even, convex, $C^2$ with globally Lipschitz Hessian, and the
classical Newton method diverges from sufficiently large initializations.
\end{corollary}

\begin{proof}
For each item, one checks that the stated $L$ has $L'(x)/L''(x)=\tilde d(x)$ for $x\neq 0$,
and that $L_+'(1)=a$ matches the choice of $a$ above; therefore \cref{thm:drift_to_loss_lipschitz}
constructs exactly the stated $L$. The regularity and divergence conclusions follow from the theorem. We indicate this in the case of $d(x)=x+x^3=x(1+x^2)$. We compute, for $x>0$,
\[
\Phi(x)=\int_1^x\frac{ds}{s(1+s^2)}
=\Big[\log s-\tfrac12\log(1+s^2)\Big]_{s=1}^{s=x}
=\log x-\tfrac12\log(1+x^2)+\tfrac12\log 2.
\]
Therefore
\[
e^{\Phi(x)}=\sqrt{2}\,\frac{x}{\sqrt{1+x^2}},
\qquad
L_+'(x)=a e^{\Phi(x)}=\frac{x}{\sqrt{1+x^2}}
\quad\text{when } a=2^{-1/2}.
\]
Integrating from $0$ to $x$ and using $L(0)=0$ yields $L(x)=\sqrt{1+x^2}-1$.
Finally, by \cref{thm:drift_to_loss_lipschitz}(ii),
\[
x_{k+1}=x_k-\tilde d(x_k)=x_k-(x_k+x_k^3)=-x_k^3,
\]
as claimed.
\end{proof}

\newpage
\subsection{A boosting/Hilbert-space realization}

\begin{proposition}[A separable empirical-risk objective with divergent Newton updates]
\label{prop:boosting_realization}
Assume that the training inputs $x_1,\dots,x_N$ are pairwise distinct.
Let $\mathcal{H}=\mathcal{L}^2(\hat\nu_N)\cong \mathbb{R}^N$ be the space of predictions on the training points,
and consider the empirical risk
\[
  L(F) := \frac1N\sum_{i=1}^N \sqrt{1 + F(x_i)^2}.
\]
Then $L$ is convex and has $2M$-Lipschitz Hessian for some $M<\infty$, and the exact Newton update satisfies, coordinate-wise,
\[
  F_{k+1}(x_i) = -F_k(x_i)^3,\qquad i=1,\dots,N.
\]
In particular, if $|F_0(x_i)|>1$ for some $i$, then $\|F_k\|\to\infty$.
\end{proposition}

\begin{proof}
The objective is separable across coordinates in $\mathbb{R}^N$. The coordinate-wise gradient and
Hessian are those of $L(x)=\sqrt{1+x^2}$, so \cref{prop:sqrt_divergence} applies componentwise.
Lipschitzness of the Hessian follows because the scalar third derivative is bounded, hence the diagonal
Hessian map is Lipschitz in operator norm.
\end{proof}

\newpage

\section{Newton Boosting for Hessian-Dominated Losses}\label{app:HessianDominatedLoss}

This appendix proves the claims stated around \cref{defHessianDominatedLoss} and \cref{prop:regularity}: binary and
categorical cross entropy are Hessian-dominated with constant $c=1$, and
pointwise Hessian-dominance implies Hessian-dominance of the empirical risk.
For the unregularized cross-entropy losses, the pointwise infimum is zero, so the pointwise bounds below are already bounds in terms of the optimality gap.

\subsection{A useful log inequality}

\begin{lemma}[A log inequality]
\label{lem:log_ineq}
For all $p\in(0,1]$,
\begin{equation}
\label{eq:log_ineq_1}
\frac{1-p}{p} \;\ge\; -\log p.
\end{equation}
Equivalently, for all $p\in[0,1)$,
\begin{equation}
\label{eq:log_ineq_2}
\frac{p}{1-p} \;\ge\; -\log(1-p).
\end{equation}
\end{lemma}

\begin{proof}
Define $\phi(p)=\frac{1}{p}-1+\log p$ on $(0,1]$. Then
\[
\phi'(p)=-\frac{1}{p^2}+\frac{1}{p}=\frac{p-1}{p^2}\le 0,
\]
so $\phi$ is non-increasing. Since $\phi(1)=0$, we have $\phi(p)\ge 0$ for all
$p\in(0,1]$, which is \eqref{eq:log_ineq_1}. Substituting $q=1-p$ gives
\eqref{eq:log_ineq_2}.
\end{proof}

\subsection{Binary cross entropy is Hessian-dominated}

\begin{proposition}[Binary cross entropy is Hessian-dominated with $c=1$]
\label{prop:binary_ce_hessdom}
Let $y\in\{0,1\}$ and $u\in\mathbb{R}$, and set $p=\sigma(u)=(1+e^{-u})^{-1}$.
Consider the binary cross entropy
\[
l(u,y) \;=\; -\Big(y\log p + (1-y)\log(1-p)\Big).
\]
Let $g(u,y)=\partial_u l(u,y)$ and $h(u)=\partial_{uu}l(u,y)$.
Then the (scalar) Newton step $f(u,y)=-h(u)^{-1}g(u,y)$ satisfies
\begin{equation}
\label{eq:binary_pointwise_dom}
\|f(u,y)\|_{h(u)}^2 := f(u,y)^2\,h(u)
= \frac{g(u,y)^2}{h(u)}
\;\ge\; l(u,y).
\end{equation}
Hence binary cross entropy is pointwise Hessian-dominated with constant $c=1$.
\end{proposition}

\begin{proof}
A direct computation gives
\[
g(u,y)=p-y,\qquad h(u)=p(1-p)>0.
\]
Thus $\|f\|_{h}^2 = g^2/h$. If $y=1$, then $g=p-1=-(1-p)$, so
\[
\frac{g^2}{h}=\frac{(1-p)^2}{p(1-p)}=\frac{1-p}{p}.
\]
Also $l(u,1)=-\log p$, and \cref{lem:log_ineq} gives
$\frac{1-p}{p}\ge -\log p=l(u,1)$. If $y=0$, then $g=p$, so
\[
\frac{g^2}{h}=\frac{p^2}{p(1-p)}=\frac{p}{1-p}.
\]
Also $l(u,0)=-\log(1-p)$, and \cref{lem:log_ineq} gives
$\frac{p}{1-p}\ge -\log(1-p)=l(u,0)$. This proves \eqref{eq:binary_pointwise_dom}.
\end{proof}

\begin{lemma}[Softmax Hessian is positive definite on the zero-sum subspace]
\label{lem:softmax_hess_posdef}
Let $p\in\mathbb{R}^K$ satisfy $p_k>0$ and $\sum_{k=1}^K p_k=1$, and define
\[
H \;:=\; \mathrm{diag}(p)-pp^\top.
\]
Let $V:=\{v\in\mathbb{R}^K:\mathbf{1}^\top v=0\}$. Then $H$ maps $V$ to $V$ and is
positive definite on $V$. In particular, the restriction $H:V\to V$ is invertible.
\end{lemma}

\begin{proof}
First, for any $v\in\mathbb{R}^K$,
\[
\mathbf{1}^\top H v
= \mathbf{1}^\top \mathrm{diag}(p)v - \mathbf{1}^\top p (p^\top v)
= p^\top v - (p^\top v) = 0,
\]
so $Hv\in V$ whenever $v\in V$. Next, for any $v\in\mathbb{R}^K$,
\[
v^\top H v
= \sum_{k=1}^K p_k v_k^2 - \Big(\sum_{k=1}^K p_k v_k\Big)^2
= \sum_{k=1}^K p_k\big(v_k-\sum_{j=1}^K p_j v_j\big)^2 \;\ge\; 0.
\]
Equality holds iff $v_k$ is constant in $k$, i.e.\ $v\in\mathrm{span}\{\mathbf{1}\}$.
Intersecting with $V$ leaves only $v=0$, hence $v^\top H v>0$ for all nonzero
$v\in V$. Therefore $H$ is positive definite on $V$ and invertible as $H:V\to V$.
\end{proof}

\begin{proposition}[Categorical cross entropy is Hessian-dominated with $c=1$]
\label{prop:categorical_ce_hessdom}
Let $u\in\mathbb{R}^K$, $p=\mathrm{softmax}(u)$, and let $y=e_t$ be a one-hot label.
Consider categorical cross entropy
\[
l(u,y) \;=\; -\sum_{k=1}^K y_k\log p_k \;=\; -\log p_t.
\]
Let $g(u,y)=\nabla_u l(u,y)=p-y$ and $H(u)=\nabla^2_{uu}l(u,y)=\mathrm{diag}(p)-pp^\top$.
Work on the subspace $V=\{v:\mathbf{1}^\top v=0\}$; then $g(u,y)\in V$ and by
\cref{lem:softmax_hess_posdef} the Newton direction $f\in V$ solving
\begin{equation}
\label{eq:cat_newton}
H(u)f \;=\; y-p
\end{equation}
is unique. Moreover,
\begin{equation}
\label{eq:cat_decrement}
\|f\|_{H(u)}^2 := f^\top H(u)f
\;=\; \frac{1-p_t}{p_t}
\;\ge\; -\log p_t \;=\; l(u,y).
\end{equation}
Hence categorical cross entropy is pointwise Hessian-dominated with constant $c=1$.
\end{proposition}

\begin{proof}
Let $H=\mathrm{diag}(p)-pp^\top$. Define $s:=p^\top f$. Expanding \eqref{eq:cat_newton} yields
\[
\mathrm{diag}(p)f - p(p^\top f)=y-p
\quad\Longleftrightarrow\quad
p_k(f_k-s)=y_k-p_k\qquad \forall k,
\]
so (since $p_k>0$) we have
\[
f_k = s + \frac{y_k}{p_k} - 1.
\]
With $y=e_t$, this becomes
\[
f_t = s + \frac{1}{p_t}-1,\qquad f_k=s-1\quad(k\neq t).
\]
Using symmetry of $H$ and \eqref{eq:cat_newton},
\[
\|f\|_H^2=f^\top H f = f^\top (y-p).
\]
Compute
\begin{align*}
f^\top (y-p)
&= f_t(1-p_t) + \sum_{k\neq t} f_k(0-p_k) \\
&= f_t(1-p_t) - (s-1)\sum_{k\neq t}p_k \\
&= \big(f_t-(s-1)\big)(1-p_t)
= \left(\left(s+\frac{1}{p_t}-1\right)-(s-1)\right)(1-p_t) \\
&= \frac{1-p_t}{p_t}.
\end{align*}
Finally, \cref{lem:log_ineq} gives $\frac{1-p_t}{p_t}\ge -\log p_t=l(u,y)$,
proving \eqref{eq:cat_decrement}.
\end{proof}

\begin{remark}[Uniqueness of the Newton step]
\label{rem:cat_uniqueness}
On $\mathbb{R}^K$ one has $H(u)\mathbf{1}=0$ (softmax is shift-invariant), so the
Newton equation is not uniquely solvable without fixing a gauge. Restricting to the
zero-sum subspace $V=\{v:\mathbf{1}^\top v=0\}$ (or equivalently fixing one logit
coordinate) yields uniqueness and satisfies \cref{assumptionHessianNewtonEquationUniqueSolution}.
The decrement $\|f\|_{H(u)}^2$ is invariant to the choice of representative.
\end{remark}

\begin{remark}
Binary cross entropy corresponds to the special case $K=2$ of categorical
cross entropy (after identifying the scalar logit with the difference of the
two softmax logits). Thus \cref{prop:categorical_ce_hessdom}
strictly generalizes \cref{prop:binary_ce_hessdom}; the latter is
included separately for clarity and because it avoids the gauge-fixing
required in the multiclass setting.
\end{remark}

\begin{lemma}\label{lem:regularized_loss_hessdom}
Suppose $l(\cdot,y)$ is twice continuously differentiable, convex,
and $S$-smooth. For $\mu>0$, let
\[
    l_\mu(u,y):=l(u,y)+\frac{\mu}{2}\|u\|^2,
    \qquad
    l_{\mu,\star}(y):=\inf_v l_\mu(v,y).
\]
Then
\[
    \nabla_1 l_\mu(u,y)^\top
    \bigl(\nabla_1^2 l_\mu(u,y)\bigr)^{-1}
    \nabla_1 l_\mu(u,y)
    \geq
    \frac{2\mu}{S+\mu}
    \bigl(l_\mu(u,y)-l_{\mu,\star}(y)\bigr).
\]
\end{lemma}

\begin{proof}
The regularized loss is $\mu$-strongly convex and
$(S+\mu)$-smooth. Writing
$g=\nabla_1 l_\mu(u,y)$ and $H=\nabla_1^2 l_\mu(u,y)$,
we therefore have
\[
    g^\top H^{-1}g
    \geq \frac{\|g\|^2}{S+\mu}
    \geq \frac{2\mu}{S+\mu}
    \bigl(l_\mu(u,y)-l_{\mu,\star}(y)\bigr).
\]
\end{proof}

Binary and categorical cross-entropy are convex and smooth,
so \cref{lem:regularized_loss_hessdom} applies to both.
Their regularized empirical risks are therefore Hessian-dominated
by \cref{lem:pointwise_to_empirical_hessdom}.
For binary cross entropy, we can prove the stronger bound
$c=(1+4\mu)/(1+2\mu)>1$, as shown in
\cref{lem:regularized_binary_ce_hessdom}.

\begin{lemma}[Regularized binary cross entropy]
\label{lem:regularized_binary_ce_hessdom}
Let
\[
    l_\mu(u,y)=\log(1+e^u)-yu+\frac{\mu}{2}u^2,
    \qquad
    l_{\mu,\star}(y)=\inf_v l_\mu(v,y),
\]
where $y\in\{0,1\}$ and $\mu>0$. Then
\[
    \frac{(\partial_u l_\mu(u,y))^2}
         {\partial_{uu}l_\mu(u,y)}
    \geq
    \frac{1+4\mu}{1+2\mu}
    \bigl(l_\mu(u,y)-l_{\mu,\star}(y)\bigr).
\]
\end{lemma}

\begin{proof}
Write $p=\sigma(u)$, $h=p(1-p)$,
$g=p-y+\mu u$, and $d=g/(h+\mu)$. Set
\[
    q=p-hd=\frac{hy+\mu(p-hu)}{h+\mu}\in(0,1).
\]
Indeed, $p-hu$ increases from $0$ to $1$, since
$(p-hu)'=-uh'\geq0$.
Set $a=\log(q/(1-q))$, so that $\partial_1 l(a,y)=q-y$.
By convexity,
\[
    l_\mu(v,y)
    \geq l(a,y)+(q-y)(v-a)+\frac{\mu}{2}v^2.
\]
The right-hand side is minimized at $v=(y-q)/\mu$, giving
\[
    l_{\mu,\star}(y)
    \geq l(a,y)-(q-y)a-\frac{(q-y)^2}{2\mu}.
\]
Using the cross-entropy expressions, completing the square, and
applying $\log x\leq x-1$, we obtain
\begin{align*}
    l_\mu(u,y)-l_{\mu,\star}(y)
    &\leq l(u,y)-l(a,y)+\frac{\mu}{2}u^2
        +(q-y)a+\frac{(q-y)^2}{2\mu}\\
    &=l(u,y)-l(a,y)-(q-y)(u-a)
        +\frac{(q-y+\mu u)^2}{2\mu}\\
    &=q\log\frac{q}{p}
        +(1-q)\log\frac{1-q}{1-p}
        +\frac{(q-y+\mu u)^2}{2\mu}\\
    &\leq q\left(\frac{q}{p}-1\right)
        +(1-q)\left(\frac{1-q}{1-p}-1\right)
        +\frac{(q-y+\mu u)^2}{2\mu}\\
    &=\frac{(q-p)^2}{h}
        +\frac{(q-y+\mu u)^2}{2\mu}\\
    &=\left(h+\frac{\mu}{2}\right)d^2,
\end{align*}
where $q-p=-hd$ and $q-y+\mu u=\mu d$.
Since $g^2/(h+\mu)=(h+\mu)d^2$ and $h\leq1/4$,
the claimed constant follows from
\[
    \frac{h+\mu}{h+\mu/2}
    \geq\frac{1+4\mu}{1+2\mu}.
    \qedhere
\]
\end{proof}

\subsection{From pointwise to empirical Hessian-dominance}

\begin{lemma}[Pointwise $\Rightarrow$ empirical Hessian-dominance]
\label{lem:pointwise_to_empirical_hessdom}
Let $\{(x_i,y_i)\}_{i=1}^N$ be a dataset with pairwise distinct inputs $x_1,\dots,x_N$, and define the empirical risk
\[
L(F)=\frac{1}{N}\sum_{i=1}^N l(F(x_i),y_i),
\]
with $F(x_i)\in\mathbb{R}^K$ (take $K=1$ for binary). We work in the empirical Hilbert
space $\mathcal{H}=\mathcal{L}^2(\hat\nu_N)$ with inner product
\[
\langle A,B\rangle = \frac{1}{N}\sum_{i=1}^N \langle A(x_i),B(x_i)\rangle_{\mathbb{R}^K}.
\]
Let $l_{\star,i}:=\inf_u l(u,y_i)>-\infty$ and $L_\star:=\inf_{F\in\mathcal H}L(F)$. At each iterate $F=F_k$, write $u_i:=F(x_i)$ and let
\[
g_{(i)} := \partial_1 l(u_i,y_i)\in\mathbb{R}^K,\qquad
h_{(i)} := \partial_{11} l(u_i,y_i)\in\mathbb{R}^{K\times K}.
\]
Assume there exists $c>0$ such that for all $i$ and $k$,
\begin{equation}
\label{eq:pointwise_dom_assump}
g_{(i)}^\top h_{(i)}^{-1} g_{(i)} \;\ge\; c\bigl(l(u_i,y_i)-l_{\star,i}\bigr),
\end{equation}
where $h_{(i)}^{-1}$ is taken on the relevant space (for categorical CE one may restrict
to $V$ pointwise as in \cref{prop:categorical_ce_hessdom}).
Let $g=\nabla L(F)$ and $H=\nabla^2 L(F)$, and let $f=-H^{-1}g$ be the exact Newton
direction. Then
\[
\|f\|_{H}^2 \;\ge\; c\bigl(L(F)-L_\star\bigr).
\]
In particular, $L$ is Hessian-dominated in the sense of \cref{defHessianDominatedLoss} with constant $c$.
\end{lemma}

\begin{proof}
In the empirical setting, the gradient and Hessian act pointwise:
\[
g(x_i)=g_{(i)},\qquad (H[f])(x_i)=h_{(i)} f(x_i).
\]
Thus the Newton equation $H[f]=-g$ decouples over samples:
\[
h_{(i)} f(x_i)=-g_{(i)} \quad\Rightarrow\quad f(x_i)=-h_{(i)}^{-1}g_{(i)}.
\]
Therefore, using \cref{defCosineAngle},
\begin{align*}
\|f\|_{H}^2
&= \langle f,H[f]\rangle
= \frac{1}{N}\sum_{i=1}^N f(x_i)^\top h_{(i)} f(x_i)
= \frac{1}{N}\sum_{i=1}^N g_{(i)}^\top h_{(i)}^{-1} g_{(i)}.
\end{align*}
Applying \eqref{eq:pointwise_dom_assump} term-by-term yields
\[
\|f\|_{H}^2
\ge \frac{c}{N}\sum_{i=1}^N \bigl(l(u_i,y_i)-l_{\star,i}\bigr)
\ge c\bigl(L(F)-L_\star\bigr),
\]
where the last inequality uses $L_\star\geq N^{-1}\sum_{i=1}^N l_{\star,i}$. This completes the proof.
\end{proof}

\section{Proof of Regularity Conditions in the Boosting Space}
\label{app:RegularityAssumptionsCarryOverProof}

Recall from \cref{subsec:GBDTsAsRestrictedConvexOptimization} that we identify the empirical boosting space $\mathcal{H} = \mathcal{L}^2(\widehat{\nu}_{X,N})$ with $\mathbb{R}^{N \times K}$, equipped with the inner product $\langle f, h\rangle = \frac{1}{N}\sum_{i=1}^N \langle f(x_i),h(x_i) \rangle_{\mathbb{R}^K}$. The induced norm is therefore $\|f\| = \big(\frac{1}{N}\sum_{i=1}^N \|f(x_i)\|^2_{\mathbb{R}^K}\big)^{1/2}$.

\vspace{5pt}
\begin{proof}[Proof of \cref{prop:regularity}]
First, we prove the claim regarding the Lipschitz continuity of the Hessian. Suppose the base loss $l(u, y)$ has a $M_0$-Lipschitz Hessian with respect to the prediction $u \in \mathbb{R}^K$, meaning $\|\nabla_{11}^2 l(u, y) - \nabla_{11}^2 l(v, y)\|_{\text{op}} \leq M_0 \|u - v\|_{\mathbb{R}^K}$. In the gradient boosting space, we minimize the empirical risk $L(F) = \frac{1}{N} \sum_{i=1}^N l(F(x_i), y_i)$. Because the Hessian operator $\nabla^2 L(F)$ acts pointwise (i.e., it is block-diagonal over the $N$ samples), its operator norm on $\mathcal{H}$ is exactly the maximum of the operator norms of the individual $K \times K$ blocks. Thus, for any $F, G \in \mathcal{H}$, we can bound the difference in Hessians as follows:
\begin{align*}
    \|\nabla^2 L(F) - \nabla^2 L(G)\|_{\text{op}, \mathcal{H}} 
    &= \max_{1 \leq i \leq N} \|\nabla_{11}^2 l(F(x_i), y_i) - \nabla_{11}^2 l(G(x_i), y_i)\|_{\text{op}, \mathbb{R}^K} \\
    &\leq \max_{1 \leq i \leq N} M_0 \|F(x_i) - G(x_i)\|_{\mathbb{R}^K} \\
    &\leq M_0 \left( \sum_{i=1}^N \|F(x_i) - G(x_i)\|_{\mathbb{R}^K}^2 \right)^{1/2} \\
    &= M_0 \sqrt{N} \left( \frac{1}{N}\sum_{i=1}^N \|F(x_i) - G(x_i)\|_{\mathbb{R}^K}^2 \right)^{1/2} \\
    &= M_0 \sqrt{N} \|F - G\|.
\end{align*}
This proves that $L$ has an $M_0\sqrt{N}$-Lipschitz Hessian in the space $\mathcal{H}$.

Next, we prove that smoothness and strong convexity are preserved with the same constants. The Hessian quadratic form in the boosting space is given by $\langle f, \nabla^2 L(F)[f] \rangle = \frac{1}{N} \sum_{i=1}^N \langle f(x_i), \nabla_{11}^2 l(F(x_i), y_i) f(x_i) \rangle_{\mathbb{R}^K}$. If the base loss $l$ is $\mu$-strongly convex and $S$-smooth, its Euclidean quadratic form is bounded pointwise:
\begin{equation*}
    \mu \|f(x_i)\|_{\mathbb{R}^K}^2 \leq \langle f(x_i), \nabla_{11}^2 l(F(x_i), y_i) f(x_i) \rangle_{\mathbb{R}^K} \leq S \|f(x_i)\|_{\mathbb{R}^K}^2.
\end{equation*}
Summing these bounds over $i=1, \dots, N$ and dividing by $N$ directly yields $\mu \|f\|^2 \leq \langle f, \nabla^2 L(F)[f] \rangle \leq S \|f\|^2$. Finally, if $l$ is pointwise Hessian-dominated, this property is similarly preserved for $L$ with the same constant $c$ via an identical pointwise summation argument, as explicitly proven in \cref{lem:pointwise_to_empirical_hessdom} in the Appendix on Hessian dominance.
\end{proof}

\end{document}